\documentclass[pdflatex,sn-mathphys-ay]{sn-jnl}

\usepackage{graphicx}%
\usepackage{multirow}%
\usepackage{amsmath,amssymb,amsfonts}%
\usepackage{amsthm}%
\usepackage{mathrsfs}%
\usepackage[title]{appendix}%
\usepackage{xcolor}%
\usepackage{textcomp}%
\usepackage{manyfoot}%
\usepackage{booktabs}%
\usepackage{algorithm}%
\usepackage{algorithmicx}%
\usepackage{algpseudocode}%
\usepackage{listings}%
\def \mcs{\mathcal{S}}
\def \mca{\mathcal{A}}

\def \mcd{\mathcal{D}}

\usepackage{bm}
\usepackage{comment}
\usepackage{diagbox}
\usepackage{appendix}
\usepackage[caption=false,font=footnotesize]{subfig}
\graphicspath{{figure/}} 
\usepackage{placeins} 
\usepackage{cancel}

\theoremstyle{thmstyleone}%
\newtheorem{theorem}{Theorem}
\theoremstyle{thmstyletwo}%

\theoremstyle{thmstylethree}%
\newtheorem{prop}{Proposition}
\newfloat{function}{t}{lop}
\floatname{function}{Function}

\begin{document}

\title[Certifying Lower Bounds for Risk-Sensitive Reinforcement Learning under Adversarial State Perturbations]{Certifying Lower Bounds for Risk-Sensitive Reinforcement Learning under Adversarial State Perturbations}


\author[1]{\fnm{Tong} \sur{Li}}\email{tli31@CougarNet.UH.EDU}

\author[2]{\fnm{Saunak Kumar} \sur{Panda}}\email{spanda@CougarNet.UH.EDU}

\author*[1]{\fnm{Yisha} \sur{Xiang}}\email{yxiang4@Central.UH.EDU}

\affil[1]{\orgdiv{Department of Industrial and Systems Engineering}, 
\orgname{University of Houston}, 
\orgaddress{\street{4302 University Dr}, \city{Houston}, \postcode{77004}, \state{TX}, \country{USA}}}

\affil[2]{\orgdiv{Department of Decision and Information Science}, 
\orgname{University of Houston}, 
\orgaddress{\street{4302 University Dr}, \city{Houston}, \postcode{77004}, \state{TX}, \country{USA}}}


\abstract{Reinforcement learning (RL) agents deployed in real-world environments are often vulnerable to adversarial perturbations in state observations, creating risks in safety-critical applications. Certification methods can improve robustness against adversarial perturbations by providing lower bounds on expected cumulative rewards. Existing certification methods, however, mainly focus on risk-neutral objectives. 
In this paper, we extend certification methods to risk-sensitive objectives by establishing lower bounds on the exponential utility of cumulative rewards under $l_{p}$-norm–bounded state adversarial perturbations ($1\leq p <\infty$).
By introducing a $\phi$-divergence relaxation of the perturbation set, we formulate the risk-sensitive certification problem as a convex optimization and derive its dual to obtain a tractable approximation of the certified lower bound. 
We further propose an empirical method that improves certified lower bounds by selecting the training risk-aversion parameter $\beta$ independently of the risk level used during evaluation.
Experiments on both OpenAI Gym environments and a machine replacement problem show that, compared to risk-neutral training, risk-averse training generally yields policies with higher certified lower bounds, particularly under larger perturbation budgets. Moreover, under both risk-neutral and risk-averse evaluation settings, increasing risk aversion during training leads to non-monotonic certification performance, where certified lower bounds initially improve but eventually decrease due to overly conservative policies.}





\maketitle

\section{Introduction}\label{sec:introduction}

In real-world deployments, reinforcement learning (RL) agents frequently observe perturbed states due to a variety of factors such as sensor noise, environmental uncertainty, and external interference. Empirical evidence has shown that many state-of-the-art RL algorithms are highly sensitive to minor changes in the input space, making the algorithms unreliable in real-world scenarios \citep{behzadan2017vulnerability, huang2017adversarial, pattanaik2017robust}. 
Even small perturbations, such as minor shifts in observed state variables due to environmental disturbances, can significantly alter the agent’s behavior \citep{zhang2020robust}. These vulnerabilities are especially critical in safety-sensitive domains, including autonomous driving, healthcare systems, and robotic control, where decision failures may lead to serious consequences \citep{shalev2016safe, sallab2017deep, liang2022efficient, he2022robust}.

To mitigate the impact of adversarial perturbations in state observations, several empirical defense mechanisms have been developed. Two widely used approaches are adversarial training and regularization-based methods.  
Adversarial training approaches can significantly enhance policy robustness by training agents on perturbed state observations during the learning process \citep{pattanaik2017robust, zhang2020robust, kos2017delving, Behzadan2017}. Regularization-based methods typically incorporate various forms of constraints into the RL optimization objective to improve the policy robustness \citep{shen2020deep, zhang2020robust}. 

While these empirical defense methods have demonstrated improved robustness against adversarial perturbations, several adaptive attacks have been proposed against these defense strategies. It is important to provide a theoretical guarantee (e.g., lower bound) for a trained policy to disrupt the repeated games between attackers and defenders, which is referred to as \textit{robustness certification} \citep{wu2022crop, Kumar2021PolicySF, mu2024reward}. 
Robustness certification was initially developed in the context of classification tasks, aiming to provide a certified lower bound on classification accuracy under bounded adversarial perturbations \citep{wu2022crop, Kumar2021PolicySF}.
However, certifying robustness in RL involves sequential decision-making, making it substantially more challenging than the one-step prediction in classification tasks.
Building upon the foundational work on robustness certification in classification, several studies have extended this framework to RL by developing methods that certify the robustness of smoothed policies with respect to cumulative rewards
\citep{wu2022crop,mu2024reward}.

Existing certification methods primarily focus on providing lower bounds for expected cumulative rewards, which represents a risk-neutral perspective. Safety-critical applications often require risk-averse evaluation that incorporates decision-makers' risk preferences into the objective function rather than optimizing expected performance. 
In this paper, we consider the setting where the observed state is adversarially perturbed within a bounded region, while the underlying true state of the RL environment remains unchanged. 
We extend robust certification to \textit{risk-sensitive} objectives, developing lower bound certificates for the expected utilities of cumulative rewards under bounded adversarial state perturbations. 
Specifically, our formulation builds upon the ReCePS framework \citep{mu2024reward} and extends it to support exponential utility for risk-averse performance evaluation. We formulate the lower bound certificate problem under $l_{p}$-norm bounded perturbations $(1\leq p < \infty)$ as a convex optimization problem through a $\phi$-divergence-based relaxation, and derive its dual formulation to obtain tractable and risk-aware lower bounds on the expected exponential utility of smoothed policies.
This approach bridges the gap between risk-neutral certification frameworks and the practical need for risk-aware performance guarantees.

We further develop a novel empirical method to improve the certified lower bounds of RL policies.
Our results show that risk-averse training generally leads to policies with higher certified lower bounds than risk-neutral training, particularly when the perturbation budget is large. Moreover, under both risk-neutral and risk-averse evaluation settings, increasing the training-time risk aversion initially enhances certified lower bound but eventually leads to degradation due to overly conservative policies. This non-monotonic relationship between the training-time risk-aversion level and the certified robustness highlights the importance of appropriately selecting the risk-aversion parameter during training to enhance the provable guarantees of RL policies.
To the best of our knowledge, this is the first effort to enhance the robust certificates of RL policies, rather than just establishing theoretical lower bounds on their expected total rewards.

The rest of the paper is organized as follows. Section \ref{sec: literature} reviews related work on adversarial perturbations and robustness certificates in RL. Section \ref{sec: model} introduces the RL environment model for MDPs with adversarial state perturbations and formulates the optimization problem for lower bound certificates. Section \ref{sec: main result} presents the convex optimization approach for solving the lower bound certification problem and the algorithms for computing the lower bounds efficiently. Section \ref{sec: experiment} presents experimental results on OpenAI environments and a practical machine replacement problem, demonstrating that appropriate risk-averse training parameters can effectively improve lower bound performance under adversarial state perturbations.

\section{Literature Review}\label{sec: literature}
In this section, we review relevant literature on RL robustness under adversarial perturbations. We first summarize empirical defense mechanisms that improve policy robustness through training strategies. We then discuss recent advances in robustness certification methods that provide theoretical guarantees for policy performance under perturbations.

\subsection{Empirical Defense Mechanisms for Adversarial Attacks}\label{2:sec2}

Various empirical defense methods have been developed to enhance the robustness of RL agents against adversarial perturbations in state observations. Adversarial training approaches are one of the most widely studied categories of defense mechanisms. \citet{pattanaik2017robust} propose a robust adversarial RL algorithm that trains agents on adversarially perturbed state observations during the learning process, demonstrating improved robustness against test-time perturbations. In their study, an adversarial attack is modeled as a perturbation that increases the probability of the agent selecting the worst possible action. 
The adversarial states are used during training to help the agent learn policies that are robust to observation noise and misperceptions.
\cite{behzadan2017vulnerability} propose an adversarial training strategy based on the deep Q-learning (DQN) framework that periodically augments the experience replay buffer with adversarially perturbed states. These perturbations are crafted to maximize the temporal-difference (TD) error, encouraging the agent to learn policies that are more robust to observation noise and adversarial manipulation.

Regularization-based methods have also been developed for improving training stability and enhancing robustness against strong adversarial attacks.
\cite{zhang2020robust} develop various regularizers such as Kullback–Leibler (KL) divergence-based, $l_2$-distance-based and hinge-loss-based regularizers for classical RL algorithms. 
All these three types of regularizers employ a minimax optimization structure solved via convex relaxation or Stochastic Gradient Langevin Dynamics.

In addition to these methods, \cite{shen2020deep} propose the Smooth Regularized RL (SR2L) framework, where the regularizer is modeled as a local smoothness penalty measuring the divergence between policy outputs at neighboring states. Depending on the policy type, this penalty is defined using either the squared $l_2$ norm for deterministic policies or Jeffrey’s divergence for stochastic policies, encouraging the policy to produce similar actions for similar inputs.

\subsection{Robustness Certification Methods}

The development of robustness certification techniques for RL has gained significant attention as a means of providing theoretical guarantees on policy performance under adversarial perturbations. 
\cite{lutjens20a} propose CARRL (Certified Adversarially-Robust Reinforcement Learning) that adapts robustness certification techniques from computer vision domain to RL domain. The method computes certified lower bounds on $Q$-values by propagating interval bounds through neural network layers and handling ReLU activations through linear approximations. CARRL defines an $\epsilon$-ball around each observed state and selects the action with the highest guaranteed worst-case $Q$-value within this bounded region, ensuring robust performance under adversarial perturbations. 

\cite{wu2022crop} propose the CROP (Certifying Robust Policies for RL) framework to provide robustness guarantees for deep RL against adversarial perturbations in state observations. 
This framework addresses two certification criteria through distinct approaches. For per-state action certification, the method creates a smoothed $Q$-function by averaging $Q$-values over Gaussian-perturbed input states, establishing Lipschitz continuity. A certified radius is then computed by analyzing the margin between the top two $Q$-values and applying inverse cumulative distribution functions to derive the maximum perturbation bound that guarantees unchanged action selection. For cumulative reward certification, they use the global smoothing method that samples multiple noisy trajectories to compute statistical lower bounds on rewards, and local smoothing that uses adaptive search to systematically enumerate possible action changes under perturbations, yielding tighter deterministic lower bounds.

Based on randomized smoothing techniques, \cite{Kumar2021PolicySF} develop a certification framework that provides robustness guarantees for cumulative rewards in RL deployment. 
In contrast to the CROP framework \citep{wu2022crop}, which supports per-state action certification and incorporates multiple smoothing strategies, their approach is specifically designed to certify the expected total reward using a single global smoothing mechanism. To support this framework, they propose an adaptive variant of the Neyman–Pearson lemma that reduces general probabilistic adversaries to deterministic ones. They further show that any such deterministic adversary can be transformed into a structured form that concentrates the entire perturbation budget on the first coordinate at the initial timestep. This theoretical reduction enables direct certification through isometric Gaussian smoothing, yielding provable lower bounds on cumulative rewards under norm-bounded adversarial attacks with adaptively allocated budgets.

More recent advances have focused on developing more general and efficient certification frameworks. 
Mu et al. consider the problem of certifying the expected cumulative reward under adversarial perturbations and formulate an optimization model that seeks a convex reformulation based on $f$-divergence measures to quantify distributional differences between original and perturbed trajectories \citep{mu2024reward}.
Unlike previous approaches that rely on threshold-based probability estimation or Lipschitz continuity bounds, ReCePS directly certifies cumulative reward lower bounds by solving dual optimization problems by adapting the $f$-divergence relaxation of the infinite-dimensional distribution space into finite-dimensional tractable convex programs. 
The method extends certification beyond $l_2$-norm to $l_1$-norm and $l_0$-norm constraints using appropriate divergence measures such as total variation distance and R\'enyi divergence.
This framework provides a unified approach for certifying robustness across different perturbation types while maintaining the theoretical guarantees of direct cumulative reward certification.

Our work contributes to the growing literature on robust certification for RL policies by extending certification guarantees to risk-sensitive settings. Specifically, we establish certified lower bounds for the exponential utility of cumulative rewards under bounded adversarial state perturbations. Experimental results show that risk-averse training generally leads to improved certified lower bounds compared to risk-neutral training, especially under stronger state adversarial perturbations, although excessive risk aversion can eventually degrade certification performance.

\section{Problem Formulation}\label{sec: model}

This section develops the robustness certification with adversarial state perturbations in the RL deployment. We first introduce the MDP under both risk-neutral and risk-sensitive frameworks, along with the smoothed RL policies with Gaussian noise. We further formulate lower bound certification problems that quantify the worst-case performance guarantees of a given policy under bounded adversarial perturbations in the state space.

\subsection{Preliminaries}
\noindent \textbf{MDP and risk-sensitive $Q$-function.}
A finite-horizon Markov decision process (MDP) is defined by the tuple 
$(\mathcal{S}, \mathcal{A}, P, r, T, \gamma)$, where $\mathcal{S} \subseteq \mathbb{R}^d$ is the continuous state space, $\mathcal{A}$ is the action space, $P: \mathcal{S} \times \mathcal{A} \rightarrow \Delta(\mathcal{S})$ is the transition kernel, $r: \mathcal{S} \times \mathcal{A} \rightarrow \mathbb{R}$ is a bounded reward function, $\gamma \in [0,1]$ is the discount factor, and $T<\infty$ is the finite time horizon.  Policy \(\pi \in \Pi\) is a mapping from the state space $\mathcal{S}$ to a probability distribution over $\mathcal{A}$, where $\Pi$ is the set of all possible policies.
At each time step \(t \in \{1, \dots, T\}\), the agent observes the current state \(s_t\), selects an action \(a_t \sim \pi(\cdot| s_t)\), receives a reward \(r(s_t, a_t)\), and transitions to the next state \(s_{t+1} \sim P(\cdot | s_t, a_t)\). 
In a risk-neutral environment, the performance of a policy is typically evaluated by the expected cumulative reward 
$\mathbb{E}_\pi\left[ \sum_{t=1}^{T} \gamma^t r(s_t, a_t) \mid s_1 = s \right]$.

RL algorithms often rely on the $Q$-function (or action-value function), defined at each time step $t\in \{1,\ldots,T\}$ as:
\begin{equation}\label{risk-neutral Q}
Q_t^\pi(s, a) = \mathbb{E}_\pi\left[ \sum_{k=t}^{T} \gamma^{k - t} r(s_k, a_k) \mid s_t = s, a_t = a \right].
\end{equation}
The \textit{optimal Q-function} is defined by $Q_t^*(s, a) = \max_{\pi \in \Pi} Q_t^\pi(s, a)$, 
which satisfies the \textit{Bellman optimality equation}:
\begin{equation*}
Q_t^*(s, a) = r(s, a) + \gamma \mathbb{E}_{s' \sim P(\cdot \mid s, a)} \left[ \max_{a' \in \mathcal{A}} Q_{t+1}^*(s', a') \right], 
\end{equation*}
for $t = 1, \dots, T-1$.
The optimal policy \(\pi^*_{t}\) is the greedy policy with respect to \(Q_t^*\),
i.e., $\pi^{*}_{t}(s) \in \arg\max_{a \in \mathcal{A}} Q_t^*(s, a)$, for all $s \in \mathcal{S}$ and $t \in \{1, \dots, T\}$.

To account for decision-makers' risk preferences, a risk-sensitive RL objective is usually needed. Among various risk-sensitive-RL modeling approaches, exponential utility of the cumulative reward is often used for its mathematical tractability and its interpretation as a robust control criterion under model uncertainty \citep{jaquette1976utility, howardrisk, Hansen2001}. 
For each state-action pair $(s, a)$, the risk-sensitive $Q$-function at time $t$ under a given policy $\pi$ is defined as
\begin{equation}\label{risk-sensitive Q}
    Q^{\pi}_{t}(s, a) := \frac{1}{\beta} \log \Big\{ \mathbb{E}_{\pi} \Big[ \exp \Big( \beta \sum_{k=t}^{T} \gamma^{k-t}r(s_t, a_t) \Big) \Big| s, a \Big] \Big\}, 
\end{equation}
where $\beta \neq 0$ denotes the risk sensitivity parameter. 
The Taylor expansion of the exponential utility function:
$$
\frac{1}{\beta} \log \mathbb{E}[\exp(\beta X)] = \mathbb{E}[X] + \frac{\beta}{2} \mathrm{Var}[X] + o(\beta)
$$
indicates that the utility function incorporates both the mean and the variance of returns. 
The sign of \(\beta\) determines whether the variance term contributes positively or negatively: negative $\beta$ implies risk aversion preference, positive $\beta$ indicates risk seeking attitude, and the $Q$-function \eqref{risk-sensitive Q} reduces to the standard risk-neutral form \eqref{risk-neutral Q} as $\beta \to 0$. Larger $|\beta|$ implies a stronger risk attitude and higher sensitivity to risk. 

Analogous to the risk-neutral case, the optimal policy \(\pi^*_{t}\) selects actions that maximize the risk-sensitive value function, i.e., $\pi^*_{t}(s) =\arg\max_{a \in \mathcal{A}} Q_t^*(s, a)$, for all $s \in \mathcal{S}$ and $t \in \{1, \dots, T\}$, where $Q_t^*(s, a) = \max_{\pi \in \Pi} Q_t^\pi(s, a)$ represents the optimal risk-sensitive $Q$-function. 
Throughout this paper, we focus on the risk-averse $Q$-function corresponding to the case where $\beta < 0$.

\noindent\textbf{Smoothed RL policy.}
The smoothed policy is defined by injecting Gaussian noise into the input state and computing the action that maximizes the expected $Q$-value over the perturbed inputs. Specifically, for each state $s_t$, a noise vector $\Delta_t \sim \mathcal{N}(0, \sigma^2 I_{d})$ is added, and the agent selects the action $a_t \in \mca$ that maximizes the smoothed $Q$-function evaluated at the perturbed state. The smoothed policy $\tilde{\pi}$ is defined as
\begin{equation}\label{eq:def smoothed policy}
\tilde{\pi}(s_t) := \pi(s_t + \Delta_t),  
\end{equation}
and the corresponding $Q$-value with the smoothed policy $\tilde{\pi}$ can be expressed as the expected $Q$-value under random input noise, and is given by
\begin{equation}
   Q^{\tilde{\pi}}(s, a)=  \mathbb{E}_{\Delta}[Q^{\pi}(s + \Delta, a)].
\end{equation}
This approach facilitates more stable policy behavior and supports certified robustness guarantees. 

\subsection{Lower bound certification}


\paragraph{Risk-sensitive RL with smoothed policy.}
We consider a finite-horizon MDP with adversarial state perturbations.
In this setting, at each time step $t$, the agent observes a perturbed state $s_t + \delta_t$, where $\delta_t \in \mathbb{R}^d$ is an additive perturbation introduced by an adversary. We assume that the perturbation sequence to time horizon $T$, $\mathbb{\delta} = (\delta_1, \ldots, \delta_{T})$ is constrained in $p$ norm, such that
$||\delta||_{p}=(\sum_{t=1}^{T} \|\delta_t\|_p^p )^{\frac{1}{p}}\leq \epsilon$, for some $ 1\leq p < \infty$ and perturbation budget $\epsilon > 0$. The set of admissible perturbations is given by
\[
\mathcal{B}^\epsilon = \left\{ \delta \in \mathbb{R}^{d \times T} : \left(\sum_{t=1}^{T} \|\delta_t\|_p^p \right)^{\frac{1}{p}} \leq \epsilon \right\}.
\]

In the context of RL with state adversarial perturbation, the action is sampled according to $a_t \sim \pi(s_{t}+\delta_{t})$, that is, the action is chosen based on perturbed state ${s}_t+\delta_{t}$. After executing $a_t$, the system transitions to the next state $s_{t+1} \sim P(\cdot|s_t, a_t)$, and the agent receives a reward $r(s_t, a_t)$.

The presence of perturbations can degrade the performance of the policy \citep{Michael2022, Kumar2021PolicySF}. The objective of robustness certification is to evaluate how well a given policy performs under worst-case perturbations. 
In this paper, we consider a risk-aversion objective and incorporate risk attitude through the exponential utility function, where the agent’s preference over returns is governed by a risk aversion level $\beta < 0$. Instead of directly evaluating the expected return, we formulate a risk-sensitive objective that accounts for worst-case perturbations by optimizing over all admissible perturbation sequences constrained by the $l_p$-norm:

\begin{equation}\label{eq:original objective}
\min_{\mathbb{\delta}} \frac{1}{\beta} \log \left(\mathbb{E}\left(\exp \beta\left(\sum_{t=1}^{T} \gamma^t r(s_t, a_t)\right)\right)\right), \quad \text{s.t. } ||\delta||_{p} \leq \epsilon.
\end{equation}
with $a_t \sim \pi(s_t + \delta_t)$.

Given a policy $\pi$, the goal of robust certification is to certify a lower bound on the cumulative reward under all feasible perturbation sequences $\delta \in \mathcal{B}^\epsilon$, 
thus providing robustness guarantees on policy performance in adversarial settings.

To facilitate analytical tractability in the lower bound analysis of \eqref{eq:original objective}, we adopt the smoothed policy $\tilde{\pi}$ defined in \eqref{eq:def smoothed policy}, which is constructed via Gaussian noise injection. 
Given the smoothed policy $\tilde{\pi}$ and initial state $s_0$, the lower bound certificate can be obtained by minimizing the exponential utility of the expected cumulative reward over the $l_{p}$ bounded adversarial perturbations $\delta=(\delta_1,\ldots,\delta_{T})$:
\begin{equation}\label{eq:risk-averse original objective}
\begin{aligned}
    &\min_{\delta}\frac{1}{\beta}\log \left(\mathbb{E}_{\tau \sim q(\tau)}\left(\exp \beta\left(\sum_{t=1}^{T}\gamma^{t}r(s_{t}, \tilde{\pi}(s_{t}+\delta_{t}))\right)\right)\right),    \\
    & \text{s.t.} \qquad ||\delta||_{p} \leq \epsilon,
    \end{aligned}
\end{equation}
where $q(\tau)$ denotes the joint distribution of the perturbed trajectory $\tau = (s_1,a_1,\ldots,s_{T}, a_{T})$ with $a_{t}\sim \tilde{\pi}(s_{t}+\delta_{t})$ and $s_{t+1}\sim P(\cdot|s_{t}, a_{t})$, for all $ t=1,\ldots,T-1$.

\paragraph{Risk-sensitive RL with stochastic observations.}
Solving the non-convex optimization problems \eqref{eq:risk-averse original objective} is generally challenging, as the objective functions involve both adversarial perturbations and stochastic smoothing noise. To address this, we adopt the approach proposed by \cite{mu2024reward} and \cite{Kumar2021PolicySF}, which considers a more general MDP setting where the agent interacts with the environment through stochastic observations. The essential idea of this approach is that the smoothed policy $\tilde{\pi}(s_t)$ can be equivalently interpreted as acting on stochastic observations sampled from a distribution \( \mu(s_t) \). 
Rather than injecting Gaussian noise \( \Delta_t \sim \mathcal{N}(0, \sigma^{2}I_{d})\) into the state and evaluating the policy on \( s_t + \Delta_t \), the same effect can be represented by sampling an observation \( o_t \in \mathcal{O}\subset \mathbb{R}^{d} \) from $\mu(\cdot|s_{t})$, where \( \mu(\cdot|s_t) \) represents the Gaussian distribution $\mathcal{N}(s_{t}, \sigma^{2}I_{d})$ and $\mathcal{O}$ denotes the continuous observation space.
The policy $\pi$ is then defined on the observation space, and the agent selects actions by sampling 
\( a_t \sim \pi(o_t) \) based on observations \( o_t \sim \mu(\cdot | s_t) \). 

To build on this observation-based formulation, we now rewrite the optimization problem in \eqref{eq:risk-averse original objective}. In this setting, the smoothed policy \( \tilde{\pi}(s_t + \delta_t) \) over the perturbed state \( s_t + \delta_t \) can be equivalently interpreted as sampling an action \( a_t \) according to \( \pi(o'_t) \), where the perturbed observation \( o'_t \) follows the Gaussian distribution \( \mu(\cdot|s_t + \delta_t) = \mathcal{N}(s_t + \delta_t, \sigma^2 I_d) \). This leads to the following optimization formulation:
\begin{equation}\label{eq:observation objective risk-averse}
    \min_{q(\tau) \in \mcd^{\epsilon}}\frac{1}{\beta}\log \left(\mathbb{E}_{\tau \sim q}\left(\exp \beta \left(\sum_{t=1}^{T}\gamma^{t}r(s_{t}, \pi(o_{t}))\right) \right)\right), 
\end{equation}
where the perturbation sequence \( \delta = (\delta_1, \ldots, \delta_{T}) \) satisfies \( \| \delta \|_p \leq \epsilon \), and each observation $o_{t}$ is sampled from the Gaussian distribution \( \mathcal{N}(s_t + \delta_t, \sigma^2 I_d) \), for all \( t = 1, \ldots, T \).
The admissible perturbation set with $p$-norm bounded constraint  ($1\leq p < \infty$) is defined as:
\begin{equation}\label{eq:original uncertainty set}
\mathcal{D}^\epsilon := \left\{ q(\tau) \in \left( \mathcal{P}(\mathcal{S}) \times \mathcal{P}(\mathcal{O}) \times \mathcal{P}(\mathcal{A}) \right)^T : \| \delta \|_p \leq \epsilon \right\}, 
\end{equation}
which represents the collection of joint distributions over trajectory \( \tau=(s_1,o_1,a_1,\ldots, s_{T}, o_{T}, a_{T}) \). Here, \( \mathcal{P}(\mathcal{S}) \), \( \mathcal{P}(\mathcal{O}) \), and \( \mathcal{P}(\mathcal{A}) \) denote the sets of probability distributions over the state, observation, and action spaces, respectively.

With a slight abuse of notation, we use \( \tau \) to denote the full trajectory consisting of states, observations, and actions under adversarial perturbations, where each observation is sampled from a Gaussian distribution \( \mu(o_t \mid s_t + \delta_t) = \mathcal{N}(s_t + \delta_t, \sigma^2 I_d) \), and each action is selected according to the policy \( a_t \sim \pi(o_t) \). Given an initial state \( s_0 \in \mathcal{S} \), and assuming $o_0 \sim \mu(o_0 \mid s_0 + \delta_0)$ and 
$a_0 \sim \pi(a_0 \mid o_0)$, the joint distribution over the trajectory 
$\tau = (s_1, o_1, a_1, \dots, s_T, o_T, a_T)$ is
\begin{equation}\label{eq: expression of q}
q(\tau) = \prod_{t=1}^{T} P(s_{t} \mid s_{t-1}, a_{t-1}) \cdot \mu(o_t \mid s_t + \delta_t) \cdot \pi(a_t \mid o_t).
\end{equation}

Note that function $\frac{1}{\beta}\log(\cdot)$ is decreasing if $\beta<0$. The minimization problem in \eqref{eq:observation objective risk-averse} can be expressed as maximization of the expected cumulative reward with the risk-averse utility measure. To facilitate solving problem \eqref{eq:observation objective risk-averse}, we reformulate it into the following standard form of convex optimization:
\begin{equation}\label{final risk-averse objective}
    \min_{q(\tau) \in \mcd^{\epsilon}}-\mathbb{E}_{\tau \sim q}\left(\exp \beta \left(\sum_{t=1}^{T}\gamma^{t}r(s_{t}, \pi(o_{t}))\right) \right), \quad \beta<0, 
\end{equation}
and the optimal value of \eqref{eq:observation objective risk-averse} can be calculated by $\frac{1}{\beta}\log(-v^{\star})$, where $v^{\star}$ denotes the optimal solution of problem \eqref{final risk-averse objective}.

In the following section, we first derive a tractable reformulation of the optimization problem \eqref{final risk-averse objective}, and then present efficient algorithms for solving the lower-bound certification problem.

\section{Reformulation of the lower bound certification}\label{sec: main result}

In this section, we first reformulate the lower bound certification problem \eqref{final risk-averse objective} into an equivalent optimization problem and derive its tractable form. We then develop algorithms to estimate the certified lower bound under state adversarial perturbations during RL deployment.

\subsection{Optimization problems of the lower bound certification}

\noindent \textbf{Convex relaxation.} Solving the optimization problem \eqref{final risk-averse objective} directly over the admissible perturbation set $\mathcal{D}^\epsilon$ \eqref{eq:original uncertainty set} is challenging due to its potential non-convex property and infinite-dimensional nature. 
In this section, we consider a convex relaxation of \(\mathcal{D}^\epsilon\) by using a \(\phi\)-divergence-based set \(\tilde{\mathcal{D}}^{\epsilon'}\), which facilitates tractable optimization and 
satisfies $\mcd^{\epsilon} \subseteq \tilde{\mcd}^{\epsilon'}$.

A $\phi$-divergence between two probability distributions $q$ and $p$ is defined as:
\begin{equation*}
    D_\phi(q \| p) = \int \phi\left( \frac{dq}{dp} \right) dp,
\end{equation*}
where $\phi: \mathbb{R}_+ \to \mathbb{R}$ is a convex function with $\phi(1) = 0$. This formulation includes many widely-used divergences, including the KL divergence, total variation (TV) distance, and the $\chi^2$-divergence \citep{bental}.

We construct the relaxed convex uncertainty set $\tilde{\mcd}^{\epsilon'}$ as a 
$\phi$-divergence ball centered at a reference distribution
\begin{equation}\label{eq: expression of p}
p(\tau) = \prod_{t=1}^{T} P(s_{t} \mid s_{t-1}, a_{t-1}) \cdot \mu(o_t \mid s_t)\cdot \pi(a_t \mid o_t),
\end{equation}
which represents the joint distribution over the trajectory $\tau$ induced by the agent's policy, transition dynamics and observation distributions. Specifically, the relaxed set is given by
\[
\tilde{\mathcal{D}}^{\epsilon'} := \left\{ q \in (\mathcal{P}(\mathcal{S}) \times \mathcal{P}(\mathcal{O}) \times \mathcal{P}(\mathcal{A}))^T : D_\phi(q \| p) \leq \epsilon' \right\},
\]
where $\epsilon'$ is the threshold determined by the perturbation budget $\epsilon$ of the constraint set $\mathcal{D}^\epsilon$ and the $\phi$-divergence function. This formulation provides a tractable relaxation of $\mathcal{D}^\epsilon$ that retains key adversarial deviations from the nominal distribution.

Therefore, based on the relaxation set $\tilde{\mcd}^{\epsilon'}$, the risk-averse certification \eqref{final risk-averse objective} can be expressed as
\begin{equation}\label{eq:reformulation objective risk-averse}
    \min_{q\in \tilde{\mcd}^{\epsilon'}}-\mathbb{E}_{\tau \sim q}\left(\exp \beta \left(\sum_{t=1}^{T}\gamma^{t}r(s_{t}, \pi(o_{t}))\right) \right), \quad \beta<0,
\end{equation}
which admits a convex formulation and can be efficiently solved using duality techniques. 
By leveraging the Lagrangian duality method \citep{mu2024reward}, the convex problem in \eqref{eq:reformulation objective risk-averse} admits an equivalent dual representation, as formalized in the following theorem.

\begin{theorem}\label{theorem 1}
Let $\tau =(s_1, o_1, a_1, \ldots, s_{T}, o_{T}, a_{T})$ denote the trajectory up to time horizon $T$. Under the $\phi$-divergence-based constraint $D_{\phi}(q||p)\leq \epsilon'$, the optimization problem in \eqref{eq:reformulation objective risk-averse} is equivalent to solving the following convex optimization problem:
\begin{equation}\label{eq: equivalent optimization form}
\max_{\xi>0,\eta\in R}\Big\{\xi\Big(\eta -\mathbb{E}_{\tau \sim p}\Big(\phi^{\star}\Big((\eta + \epsilon'+ \frac{\exp \beta \left(\sum_{t=1}^{T}\gamma^{t}r(s_{t}, \pi(o_{t}))\right)}{\xi}\Big)\Big)\Big\},
\end{equation}
where $\phi^{*}(x)=\max_{y>0}(xy - \phi(y))$ denotes the conjugate of $\phi$, 
and $\phi$ is a convex function satisfying $\phi(1)=0$.
\end{theorem}

\begin{proof}
For each trajectory $\tau =(s_1, o_1, a_1, \ldots, s_{T}, o_{T}, a_{T})$, we define $R(\tau):=\exp \beta (\sum_{t=1}^{T}\gamma^{t}r(s_{t}, \pi(o_{t})))$ as the exponential utility of the cumulative reward with risk-aversion level $\beta<0$. Therefore, the optimization problem \eqref{eq:reformulation objective risk-averse} can be expressed as 
\begin{equation}\label{eq:rewrite the objective}
    \min_{q} -\mathbb{E}_{\tau \sim q}[R(\tau)], \quad \text{s.t.}\, D_{\phi}(q||p) \leq \epsilon'.
\end{equation}
The Lagrangian dual of the optimization problem \eqref{eq:rewrite the objective} is given by:
\begin{equation}
\begin{aligned}
    & \max_{\xi>0}\min_{q}\left\{ -E_{\tau \sim q}[R(\tau)] + \xi(D_{\phi}(q||p)-\epsilon')\right\} \nonumber \\
    & = \max_{\xi>0} \left\{\xi \left[ \min_{q} \left[-E_{\tau\sim q}[\frac{R(\tau)}{\xi})] + D_{\phi}(q||p)-\epsilon' \right] \right]\right\} \nonumber \\
    &\refstepcounter{equation} 
      \overset{(\theequation)}{=}
      \label{theo1: 1}  \max_{\xi>0}\left\{ \xi \left[ \max_{t\in \mathbb{R}} \left\{t-E_{\tau \sim 
    p}[\phi^{*}(t+\frac{R(\tau)}{\xi})]-\epsilon' \right\} \right]\right\} \\
    & = \max_{\xi>0, \eta \in \mathbb{R}}\left\{ \xi \left[\eta +\epsilon' -E_{\tau\sim p}[\phi^{*}(\eta + \epsilon' + \frac{R(\tau)}{\xi})]-\epsilon' \right] \right\} \nonumber \\
    & = \max_{\xi>0, \eta\in \mathbb{R}}\left\{ \xi \left[ \eta  -E_{\tau\sim p}[\phi^{*}(\eta + \epsilon' + \frac{R(\tau)}{\xi})]  \right] \right\} \nonumber,
\end{aligned}
\end{equation}
where (\ref{theo1: 1}) is due to the duality formula of the optimization with $\phi$ divergence \citep{bental2007}:
$\min_{q}\{ -\mathbb{E}_{\tau \sim q}[R(\tau)]+D_{\phi}(q||p) \} = \max_{\eta \in \mathbb{R}}\{\eta - \mathbb{E}_{\tau \sim p}[\phi^{*}(\eta + R(\tau))] \}$.

Therefore, minimizing the optimization problem \eqref{eq:reformulation objective risk-averse} with risk-aversion level $\beta<0$ can be solved by maximizing its dual reformulation \eqref{eq: equivalent optimization form}. The dual problem is a finite-dimensional convex program involving two variables, $\xi > 0$ and $\eta \in \mathbb{R}$, where the objective is defined via the convex conjugate $\phi^*$ of the divergence function $\phi$. 
\end{proof}

Theorem \ref{theorem 1} provides a tractable reformulation of the optimization problem \eqref{eq:reformulation objective risk-averse} that holds for any choice of $\phi$-divergence. Given a reference distribution $p$,
the original infinite-dimensional problem \eqref{eq:reformulation objective risk-averse} can be equivalently reformulated as the finite-dimensional convex program \eqref{eq: equivalent optimization form}. 

In practice, selecting an appropriate $\phi$-divergence is crucial, as it ensures the validity and tractability of the convex reformulation in Theorem \ref{theorem 1}. In our setting, the uncertainty set $\mathcal{D}^\epsilon$ includes trajectory distributions induced by $l_p$-norm bounded state perturbations.
To ensure a valid relaxation, the $\phi$-divergence should induce a divergence ball $\tilde{\mathcal{D}}^{\epsilon'}$ that contains $\mathcal{D}^\epsilon$, where $\epsilon'$ is a divergence threshold determined by the original perturbation budget $\epsilon$ and the choice of $\phi$. 
The norm $p$ determines the geometric structure of $\mathcal{D}^\epsilon$, thereby guiding the selection of an appropriate divergence $\phi$ and threshold $\epsilon'$ to guarantee the inclusion $\mathcal{D}^\epsilon \subseteq \tilde{\mathcal{D}}^{\epsilon'}$.

Due to the sequential dependencies in trajectory distributions, it is challenging to directly derive the relationship between $\phi$-divergence of joint distributions and the $l_p$-norm bounded perturbations sequences $\delta = (\delta_1, \ldots, \delta_{T})$. To address this, we adopt the assumption that adversarial perturbation occurs only at the initial state \citep{mu2024reward, Kumar2021PolicySF}. 
This assumption is justified as the initial perturbation affects the future observations and actions along the trajectory, thereby capturing long-term adversarial effects while allowing for tractable divergence control. 
Under this assumption, given the initial perturbation $\delta_1 \in \mathcal{B}^{\epsilon}$ and the initial state $s_1$, the $\phi$-divergence between the distributions $q(\tau)$ and $p(\tau)$ reduces to the  $\phi$-divergence the observation distributions $\mu(\cdot | s_1)$ and $\mu(\cdot | s_1 + \delta_1)$. The detailed derivation is provided in Appendix \ref{appendix: observation reduction}.

Based on the Gaussian smoothing formulation described in \eqref{eq:def smoothed policy}, the distributions of the initial observation can be modeled as $\mu(\cdot| s_1) = \mathcal{N}(s_1, \sigma^2 I_d)$ and $\mu(\cdot| s_1+\delta_1) = \mathcal{N}(s_1+\delta_1, \sigma^2 I_d)$, respectively. Under this formulation, we identify suitable 
$\phi$-divergences for different values of $p$ ($1\leq p < \infty$), such that $D_{\phi}(\mu(o_1|s_1+\delta_1)||\mu(o_1|s_1)) \leq \epsilon', \forall ||\delta_{1}||_{p} \leq \epsilon$.
That is, given a perturbation budget $\epsilon$, one can select appropriate $\phi$-divergence functions and threshold $\epsilon'$ to ensure that $\mcd^{\epsilon}\subseteq \tilde{\mcd}^{\epsilon'}$.

The following proposition characterizes the divergence between the observation distributions $\mu(\cdot|s_1+\delta_1)$ and $\mu(\cdot|s_1)$, under adversarial perturbations bounded in $l_1$-norm. 
Under the Gaussian smoothing formulation, it establishes a relationship between the perturbation budget $\epsilon$ and the corresponding TV threshold $\epsilon'$, enabling a tractable convex relaxation of the original constraint set $\mathcal{D}^\epsilon$.

\begin{prop}\label{prop1}
Given the initial state $s_1\in \mcs$ and two Gaussian distributions $\mu(\cdot|s_1+\delta_1)=\mathcal{N}(s_1+\delta_1,\sigma^{2}I_{d})$ and $\mu(\cdot|s_1) = \mathcal{N}(s_1,\sigma^{2}I_{d})$, for any adversarial perturbation $||\delta_1||_{1}\leq \epsilon$, we have 
\begin{equation}\label{eq: TV relationship}
    D_{TV}(\mu(o_1|s_1+\delta_1)||\mu(o_1|s_1)) \leq 2 \Phi\left(\frac{\epsilon}{2 \sigma}\right) -1, 
\end{equation}
where $\Phi$ is the cumulative distribution function (CDF) of the standard normal distribution $\mathcal{N}(0,1)$.
\end{prop}
\begin{proof}
See details in Appendix \ref{appendix: proof of l1}.
\end{proof}

Given the original constraint set $\mathcal{D}^{\epsilon}$ with an $l_1$-norm bounded budget, 
Proposition \ref{prop1} shows that when the $\phi$-divergence is defined as the TV distance, 
i.e., $\phi(x) = \tfrac{1}{2} |x - 1|$, the relaxed constraint set 
$\tilde{\mathcal{D}}^{\epsilon'}$ with perturbation budget 
$\epsilon' = 2\Phi\!\left(\tfrac{\epsilon}{2\sigma}\right) - 1$ 
satisfies $\mathcal{D}^{\epsilon} \subseteq \widetilde{\mathcal{D}}^{\epsilon'}$.

It has been shown that for any $||\delta_1||_{2} \leq \epsilon$, the provably tight relaxation $\tilde{\mcd}^{\epsilon'}$ can be selected as 
\begin{equation}
    \tilde{\mcd}^{\epsilon'}=\{ \mu(\cdot|s_1+\delta_1) : D_{HS,\theta}(\mu(\cdot|s_1+\delta_1||\mu(\cdot|s_1))) \leq \epsilon'\}, \forall \theta \geq 0,
\end{equation}
where $D_{HS,\theta}(\mu(\cdot|s_1+\delta_1||\mu(\cdot|s_1)))$ denotes the Hockey-Stick divergence between Gaussian distributions $\mu(\cdot|s_1+\delta_1)$ and $\mu(\cdot|s_1)$ with $\phi(x)=\max(x-\theta,0)-\max(1-\theta,0)$ \citep{dvijotham2020framework, mu2024reward} and the threshold $\epsilon'$ is given by \citep{balle2018improving}
\begin{equation}\label{relationship hs divergence}
   \Phi\left(\frac{\epsilon}{2\sigma} - \frac{\log(\theta)}{2\epsilon} \right) - \lambda \Phi\left( -\frac{\epsilon}{2\sigma} - \frac{\log(\theta)}{2\epsilon} \right) - \max(1 - \theta, 0).
\end{equation}

In the subsequent subsection, we develop algorithms to solve the maximization problem \eqref{eq: equivalent optimization form}, aiming to certify a lower bound for RL deployment under state adversarial perturbations.

\subsection{Lower bound certificate algorithms}

In this subsection, we develop algorithms to approximate the lower bound certificate defined in \eqref{eq:reformulation objective risk-averse} under a convex relaxation set $\tilde{\mcd}^{\epsilon'}$. Theorem \ref{theorem 1} provides a tractable dual formulation \eqref{eq: equivalent optimization form} for problem \eqref{eq:reformulation objective risk-averse} by introducing the Lagrange multipliers $\xi$ and $\eta$. The algorithm for solving the dual problem \eqref{eq: equivalent optimization form} is detailed in Algorithm \ref{al: lower bound certificate}.

To solve the convex optimization problem in \eqref{eq: equivalent optimization form}, it is necessary to evaluate the expectation 
\begin{equation}\label{2: expectation}
\mathbb{E}_{\tau \sim p}(\phi^{\star}((\eta + \epsilon'+ \frac{R(\tau)}{\xi})).
\end{equation}
It is challenging to calculate the exact value of the expected return due to the sequential nature of RL. We adopt a Monte Carlo sampling method to approximate the expectation of the exponential utility return.
Specifically, 
the Monte Carlo estimate of the above expectation is given by
$
\hat{\mu} = \frac{1}{M} \sum_{i=1}^{M} 
\phi^\ast(\eta + \epsilon' + \frac{R(\tau_i)}{\xi})$, 
where \(\tau_1, \dots, \tau_M\) are independent trajectory samples drawn from \(p(\tau)\).

To quantify the estimation uncertainty, we construct a confidence interval for the true expectation at a confidence level of \(\alpha = 0.01\). 
According to Hoeffding’s inequality \citep{vershynin2018high}, given \(M\) independent samples, the expectation \eqref{2: expectation} is bounded above by
\begin{equation}\label{upper bound}
\hat{\mu}_{\text{upper}} =
\hat{\mu} + K\sqrt{\frac{\ln(2/\alpha)}{2M}}
\end{equation}
with probability at least \(1 - \alpha\), where $K$ denotes the range of $\phi^\ast(\eta + \epsilon' + \frac{R(\tau)}{\xi})$. In this paper, we adopt the upper confidence bound $\mu_{\text{upper}}$ as a conservative estimation of expectation \eqref{2: expectation}, ensuring that the resulting certificate remains a valid lower bound on policy performance with high probability.

Algorithm \ref{al: lower bound certificate} performs this estimation procedure and consists of two main components: solving the convex optimization problem in \eqref{eq: equivalent optimization form} and evaluating the exponential utility
of the cumulative reward under a smoothed policy $\tilde{\pi}$. The overall framework of Algorithm \ref{al: lower bound certificate} is adapted from \cite{mu2024reward} with modifications to accommodate our risk-sensitive certification setting.
These procedures are detailed in Functions \ref{al: optimization} and \ref{al: getreward}, respectively.

\begin{algorithm} 
\caption{Lower bound certificate algorithm} 
\label{al: lower bound certificate} 
\begin{algorithmic}[1]
\Require Trained smoothed policy $\tilde{\pi}$, smoothed $Q$-network $Q^{\tilde{\pi}}$, $\phi$-divergence-based convex relaxation set $\tilde{\mcd}^{\epsilon'}$, risk aversion level $\beta$, sample sizes $N$ and $M$ ($N < M$)
\State $\xi^{*}, \eta^{*} \leftarrow$ Optimization$(N,Q^{\tilde{\pi}},\epsilon',\phi)$

// Calculate the dual optimization problem \eqref{eq:optimization in algorithm 3} with a relatively small sample size $N$

\State $R=\{R_1,\cdots,R_M\} \leftarrow$ GetReward$(M,Q^{\tilde{\pi}})$

// Calculate the exponential utility of the cumulative reward with smoothed policy with a relatively large sample size $M$

\State Compute $\phi^{*}(\eta^{*}+\epsilon'+\frac{R_{i}}{\xi^{*}})$ for each sample $i=1,\ldots, M$, and obtain the Monte Carlo estimate $\hat{\mu}=\frac{1}{M}\sum_{i=1}^{M}\phi^{*}_{i}(\eta^{*}+\epsilon'+\frac{R_{i}}{\xi^{*}})$ along with its upper bound $\hat{\mu}_{\text{upper}}$ \eqref{upper bound}
\end{algorithmic}
\Return Certified lower bound $\underline{R}= \xi^{*}(\eta^{*} - \hat{\mu}_{\text{upper}}) $
\end{algorithm}

\begin{function}
\caption{Optimization$(N, Q^{\tilde{\pi}}, \epsilon', \phi)$} \label{al: optimization}
\begin{algorithmic}[1]
\Require Sample size $N$, trained smoothed policy $\tilde{\pi}$, smoothed $Q$-network $Q^{\tilde{\pi}}$
\State $R=\{R_1,\ldots,R_{N}\} \leftarrow \text{GetReward}(N, Q^{\tilde{\pi}})$
\State Solving the convex optimization problem 
\begin{equation}\label{eq:optimization in algorithm 3}
    \max_{\xi>0,\eta\in R}\big\{\xi\big[\eta -\sum_{i=1}^{N}\big(\phi^{\star}(\eta + {\epsilon'}+ \frac{R_{i}}{\nu}\big)\big]\big  \}
\end{equation}
\end{algorithmic}
\Return Optimal solutions $\xi^{*}, \eta^{*}$
\end{function}

In Algorithm \ref{al: lower bound certificate}, Step 1 calls \textbf{Optimization$(N, Q^{\tilde{\pi}},\epsilon',\phi)$} (Function \ref{al: optimization}) to solve the dual optimization problem defined in \eqref{eq:optimization in algorithm 3}. 
In this function, the expectation in \eqref{2: expectation} is approximated using a Monte Carlo average over $N$ return samples, which are obtained by executing the smoothed policy $\tilde{\pi}$ on perturbed trajectories (see details in Function \ref{al: getreward}). Based on this approximation, the original risk‑sensitive certification problem \eqref{eq: equivalent optimization form} is reformulated as the convex optimization problem \eqref{eq:optimization in algorithm 3}, which can be efficiently solved to obtain the optimal Lagrange multipliers $\xi^{*}$ and $\eta^{*}$. 
The optimization problem \eqref{eq:optimization in algorithm 3} is implemented using the CVXPY modeling framework \citep{diamond2016cvxpy}, which enables efficient and reliable solution of the convex objective using standard convex solvers under disciplined convex programming rules.

In step 2, Algorithm \ref{al: lower bound certificate} calls \textbf{GetReward$(M, Q^{\tilde{\pi}})$} (Function \ref{al: getreward}) to generate return samples $R=\{R_1,\dots,R_M\}$ by executing the smoothed policy $\tilde{\pi}$ over $M$ independent trajectories, each perturbed by Gaussian noise $\Delta \sim \mathcal{N}(0, \sigma^2 I_d)$. At each time step $t$, the agent selects actions according to the well-trained smoothed $Q$-network $Q^{\tilde{\pi}}$, and the exponential utility of the cumulative reward is recorded to form the Monte Carlo sample set.
Note that a relatively small sample size $N$ is used in Step 1 to improve computational efficiency when solving the optimization problem \eqref{eq:optimization in algorithm 3}, as each sample increases the complexity of the convex objective. Step 2 adopts a larger sample size $M$ to generate return samples for a more accurate Monte Carlo estimate of the certification objective in Step 3.

\begin{function} 
\caption{GetReward$(M, Q^{\tilde{\pi}})$} 
\label{al: getreward} 
\begin{algorithmic}[1]
\Require Sample size $M$, trained smoothed policy $\tilde{\pi}$, smoothed $Q$-network $Q^{\tilde{\pi}}$, risk aversion level $\beta$
\For
{$m \leftarrow 1, \ldots M$}
\State $R_{m} \leftarrow 0$
\For
{$t$ from 1 to $T$}
\State Sample random noise $\Delta \sim \mathcal{N}(0, \sigma^{2}I_{d})$
\State Inject random noise to the state $s'_{t}\leftarrow s_{t}+\Delta$
\State Select greedy action $a_{t} \leftarrow \arg \max_{a\in A}Q^{\tilde{\pi}}(s'_{t}, a)$
\State Obtain reward $r_{t}$ and next state $s_{t+1}$
\State $R_{m} \leftarrow \exp \beta(R_{m} + r_{t})$
\EndFor
\EndFor
\end{algorithmic}
\Return $R= \{ R_{1}, \ldots, R_{M}\}$
\end{function}

Finally, in Step 3, Algorithm \ref{al: lower bound certificate} computes the Monte Carlo estimate of the dual objective using $M$ sampled trajectories and derives its upper confidence bound $\hat{\mu}_{\mathrm{upper}}$ under a confidence level of $\alpha$. The certified lower bound is then computed as
\[
R = \xi^*(\eta^* - \hat{\mu}_{\mathrm{upper}}),
\]
and returned as the final output of the algorithm. This certification procedure provides a computationally tractable estimate of the certified performance guarantee under state adversarial perturbations during RL deployment.

\section{Numerical Study}\label{sec: experiment}
In this section, we conduct experiments to demonstrate the proposed methods for computing the risk-averse lower bound certification \eqref{eq:observation objective risk-averse}.
We also investigate how different levels of risk aversion during training affect the certified performance of RL policies under adversarial state perturbations.
We treat the risk aversion level $\beta$ as a tunable hyperparameter, allowing the training value ($\beta_{\text{train}}$) to differ from the evaluation-time value ($\beta_{\text{test}}$) used for lower bound certification.

\begin{figure*}[htbp]
    \centering
    
    \subfloat[$\beta_{\text{test}} \to 0$ (risk-neutral), $l_2$-norm]{%
        \includegraphics[width=0.45\textwidth]{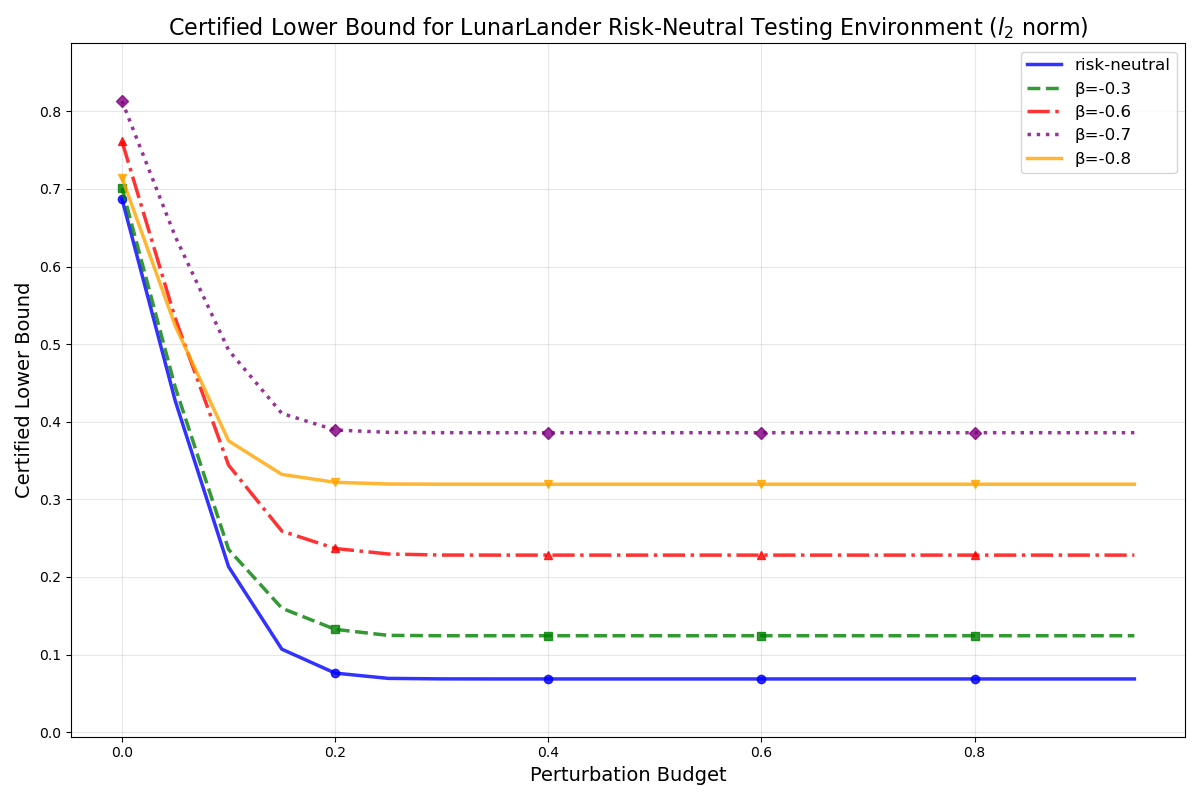}
        \label{fig: ln_risk_neutral_l2}%
    }
    \hfill
    \subfloat[$\beta_{\text{test}}=-0.2$ (risk-averse), $l_2$-norm ]{%
        \includegraphics[width=0.45\textwidth]{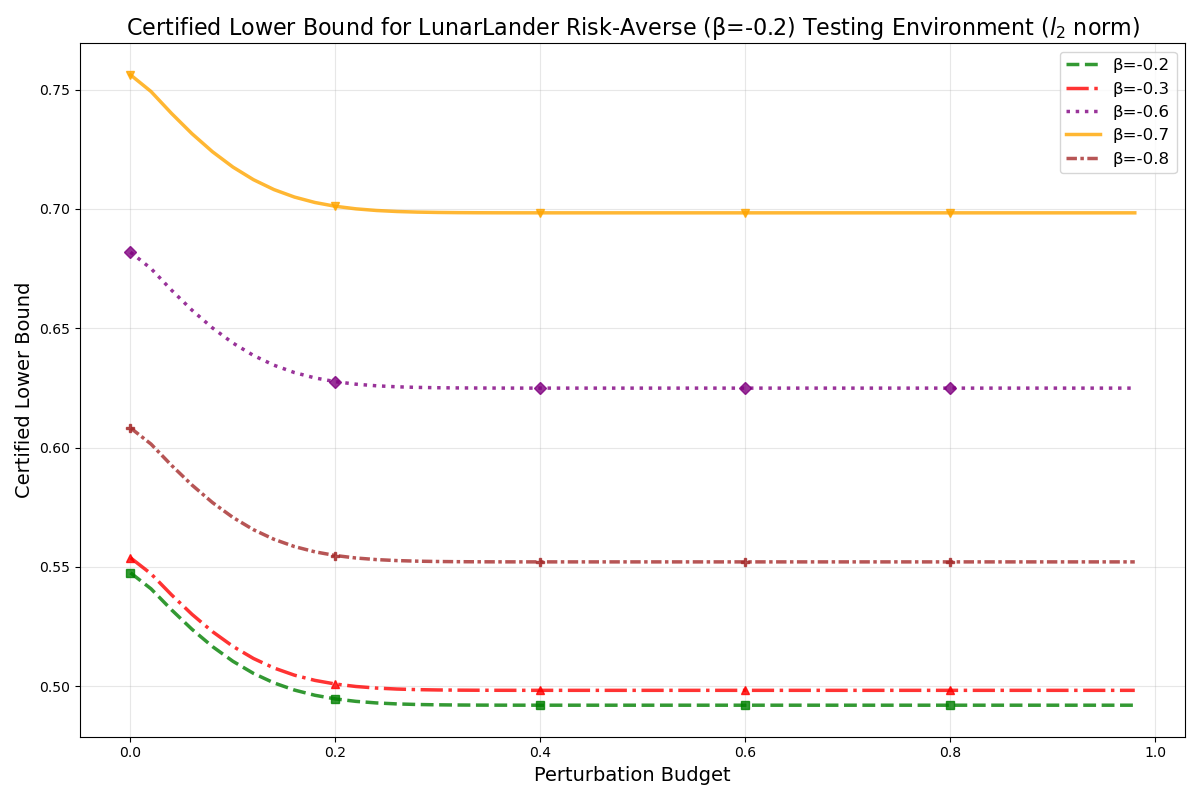}%
        \label{fig: ln_risk_averse_0.2_l2}%
    }


    \subfloat[$\beta_{\text{test}}=-0.4$ (risk-averse), $l_2$-norm]{%
        \includegraphics[width=0.45\textwidth]{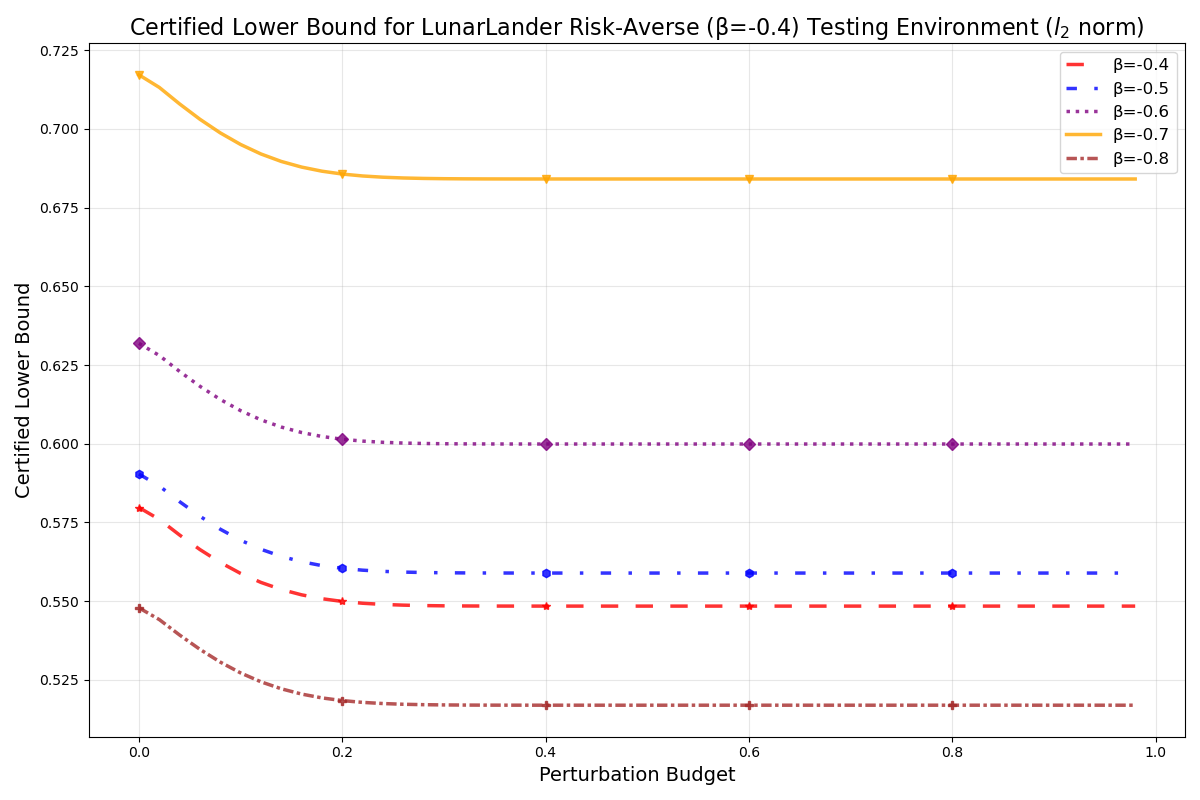}%
        \label{fig:ln_risk_averse_0.4_l2}%
    }
    \hfill
    \subfloat[$\beta_{\text{test}}=-0.6$ (risk-averse), $l_2$-norm] {%
        \includegraphics[width=0.45\textwidth]{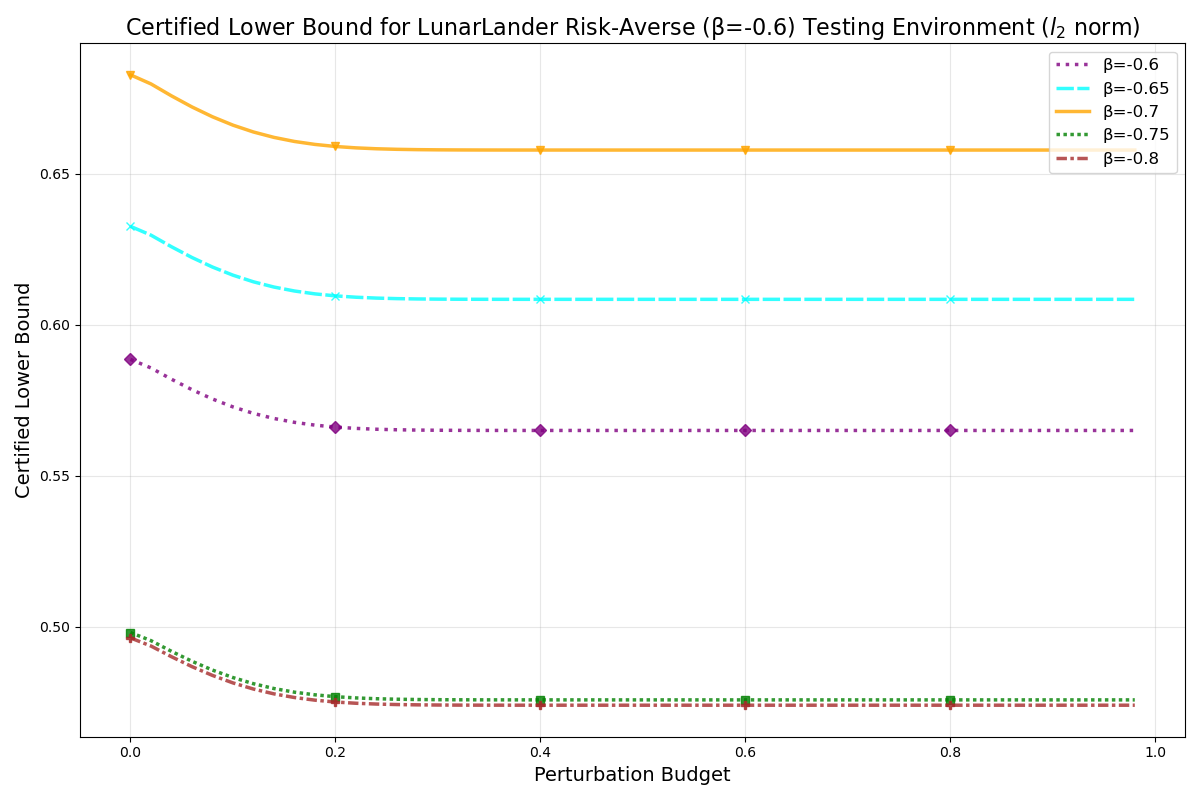}%
        \label{fig: ln_risk_averse_0.6_l2}%
    }


    \subfloat[$\beta_{\text{test}} \to 0$ (risk-neutral), $l_1$-norm]{%
        \includegraphics[width=0.45\textwidth]{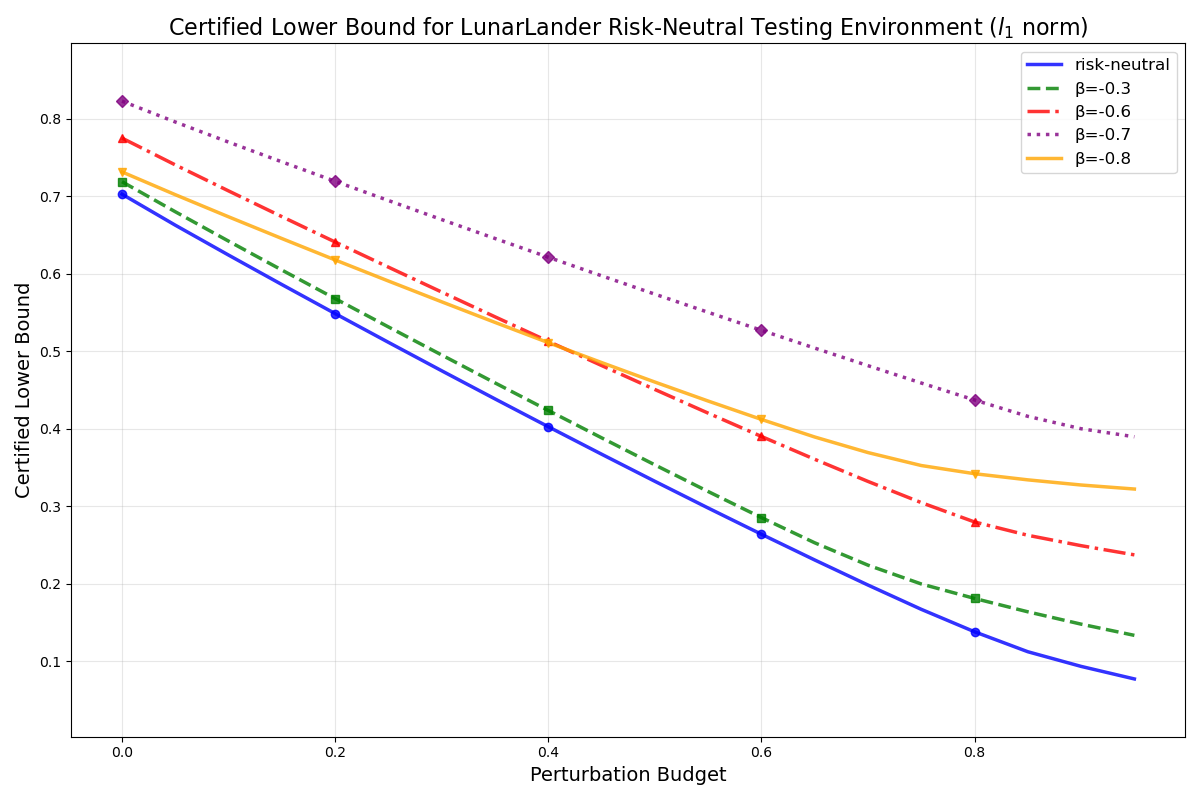}%
        \label{fig: ln_risk_neutral_l1}%
    }
    \hfill
    \subfloat[$\beta_{\text{test}}=-0.2$ (risk-averse), $l_1$-norm]{%
        \includegraphics[width=0.45\textwidth]{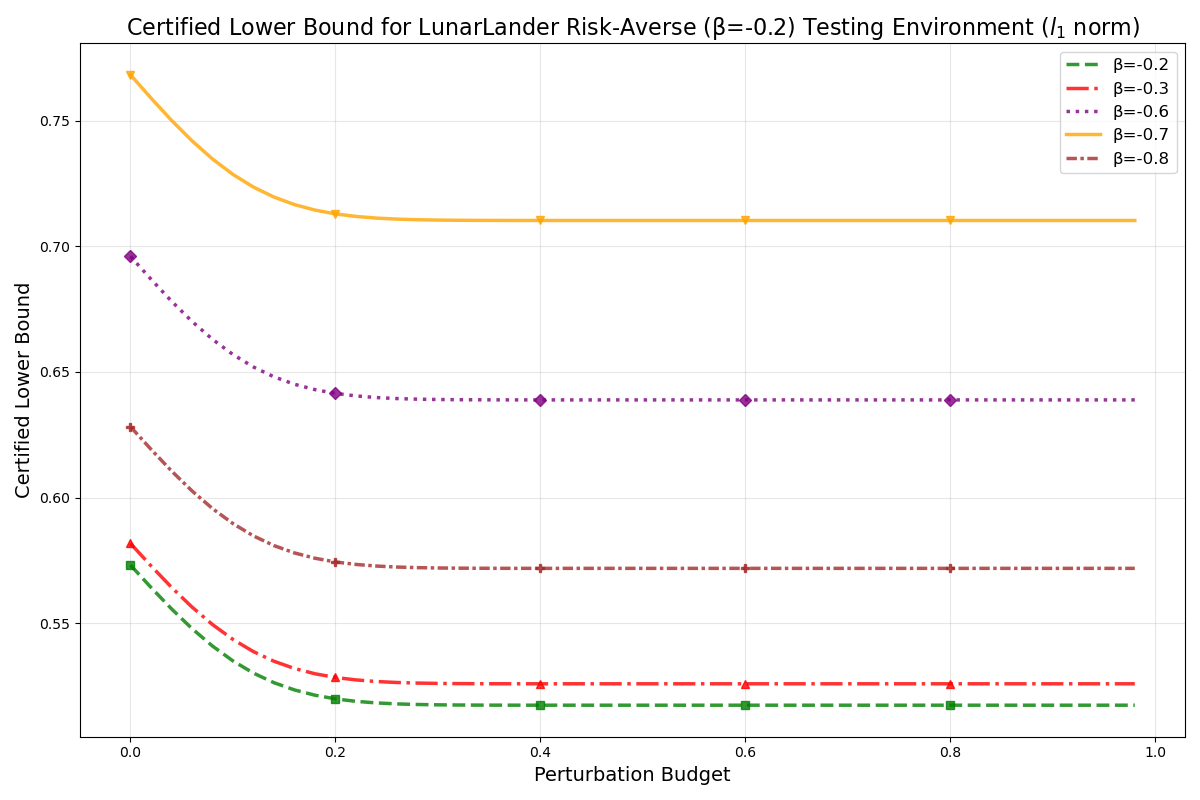}%
        \label{fig: ln_risk_averse_0.2_l1}%
    }


    \subfloat[$\beta_{\text{test}}=-0.4$ (risk-averse), $l_1$-norm]{%
        \includegraphics[width=0.45\textwidth]{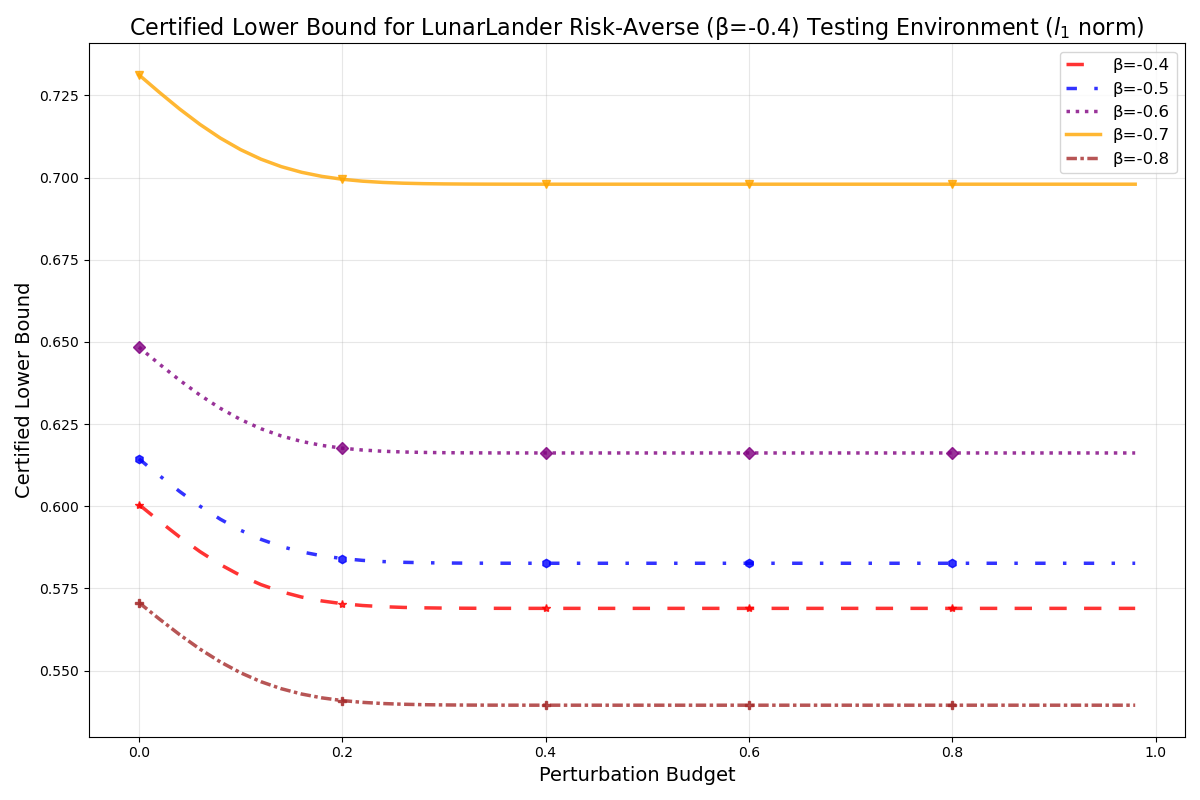}%
        \label{fig: ln_risk_averse_0.4_l1}%
    }
    \hfill
    \subfloat[$\beta_{\text{test}}=-0.6$ (risk-averse), $l_1$-norm]{%
        \includegraphics[width=0.45\textwidth]{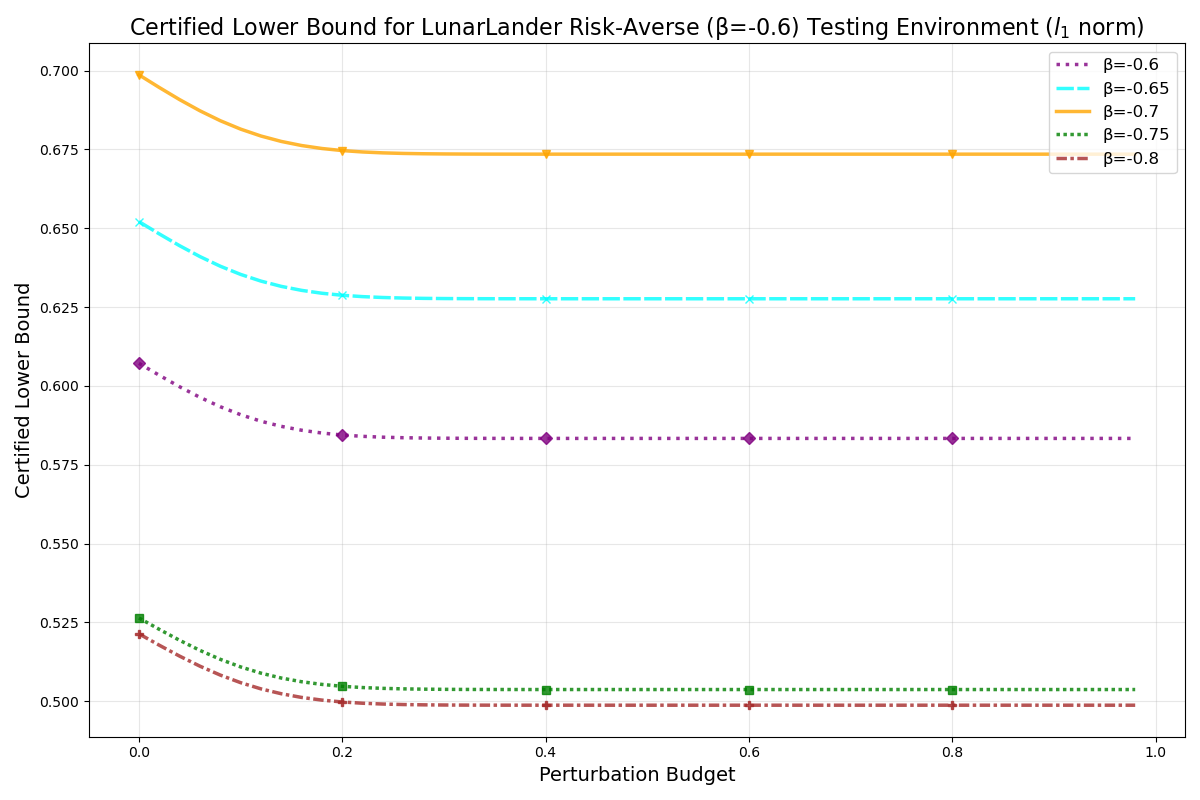}%
        \label{fig: ln_risk_averse_0.6_l1}%
    }

    \caption{Certified lower bounds for Lunar Lander with risk-neutral and risk-averse testing environments under $l_{2}$-norm (top two rows) and $l_1$-norm (bottom two rows) perturbations.}
    \label{fig: lunar lander result}
\end{figure*}

We train multiple policies using the risk-neutral objective \eqref{risk-neutral Q} and the risk-averse objective \eqref{risk-sensitive Q} with varying values of $\beta_{\text{train}}$. These policies are then evaluated under both risk-neutral ($\beta_{\text{test}} \to 0$) and risk-averse ($\beta_{\text{test}} < 0$) testing-time objectives \eqref{eq:observation objective risk-averse}, and their certified lower bounds are compared to assess policy performance across different deployment conditions.
The experiments are designed to identify the training-time risk aversion level $\beta_{\text{train}}$ that maximizes certified lower bound for each testing deployment condition. The results suggest that tuning $\beta_{\text{train}}$ can influence certified performance, which provides practical guidance for selecting the suitable training hyperparameters based on anticipated deployment scenarios.

We first conduct experiments in the Lunar Lander and CartPole environments from OpenAI Gym, both of which feature continuous state spaces and discrete action spaces \citep{towers2024gymnasium}.
We then extend our analysis to a stylized machine replacement problem \citep{puterman1994markov}, with the goal of providing interpretability into how risk-averse training environments contribute to the improvement of lower bound certificates under state adversarial perturbations.
To prevent computational issues due to vanishing exponential values arising from large rewards under negative risk aversion parameters 
$\beta$, we normalize the reward scale across all three environments.

\subsection{OpenAI Gym environment}

We first consider the Lunar Lander environment from the Box2D suite of OpenAI Gym, where the agent controls a spacecraft to land on a designated pad \citep{towers2024gymnasium}. Since small variations in state observations (e.g., position or velocity) can influence the selection of actions and ultimately affect task outcomes, this environment is well-suited for evaluating the certified lower bounds for RL policies under state adversarial perturbations.
In particular, we apply Gaussian noise with zero mean and a standard deviation of 0.05 to simulate perturbations in the continuous state observations. For a range of perturbation budgets $\epsilon \in (0, 1)$, we compute the corresponding certified lower bounds \eqref{eq:observation objective risk-averse} under $l_2$- and $l_1$-norm constraints to assess the performance of different RL policies.
We evaluate the trained policies across four testing risk levels, including a risk-neutral certification \eqref{eq:observation objective risk-averse} with $\beta_{\text{test}} \to 0$, and three risk-averse certifications \eqref{eq:observation objective risk-averse} with $\beta_{\text{test}} = {-0.2, -0.4, -0.6}$. 
For each testing scenario, we evaluate and compare the certified lower bounds of RL policies trained under different training objectives and risk aversion levels.

Figure \ref{fig: lunar lander result} presents the certified lower bounds of RL policies evaluated under multiple testing risk aversion levels in the Lunar Lander environment. The results include both risk-neutral and risk-averse testing scenarios, where certified lower bounds are evaluated under adversarial perturbations bounded by the $l_2$ norm (top two rows) and $l_1$ norm (bottom two rows).
Each curve represents the certified lower bound of an RL policy trained with a distinct $\beta_{\text{train}}$ value.
In all subfigures, we observe that the certified lower bounds are non-increasing as the perturbation budget increases.

Figures \ref{fig: ln_risk_neutral_l2} and \ref{fig: ln_risk_neutral_l1} show the certified lower bounds of policies trained under varying levels of risk aversion, evaluated in a risk-neutral testing environment with $l_2$- and $l_1$-norm bounded adversarial perturbations, respectively.
We select the curve corresponding to the policy trained with the risk-neutral objective as the baseline for comparison. 
We observe that risk-averse training generally produces policies with higher certified lower bounds than those obtained through risk-neutral training, particularly under larger perturbation budgets. This suggests that incorporating risk aversion during training enhances the certified robustness of policies in risk-neutral evaluation settings with adversarial perturbations.
Moreover, we also find that as the risk aversion parameter $\beta_{\text{train}}$ decreases, the certified lower bounds initially increase and then decrease. Specifically, as $\beta_{\text{train}}$ decreases from $-0.3$ to $-0.7$, the corresponding lower bound curves gradually improve, but at $\beta_{\text{train}} = -0.8$, the lower bound decreases, exhibiting a non-monotonic trend. This indicates that excessive risk aversion during training can lead to over conservative policies and consequently degrade certified performance under perturbations.

\begin{figure*}[htbp]
    \centering
    
    \subfloat[$\beta_{\text{test}} \to 0$ (risk-neutral), $l_{2}$-norm]{%
        \includegraphics[width=0.45\textwidth]{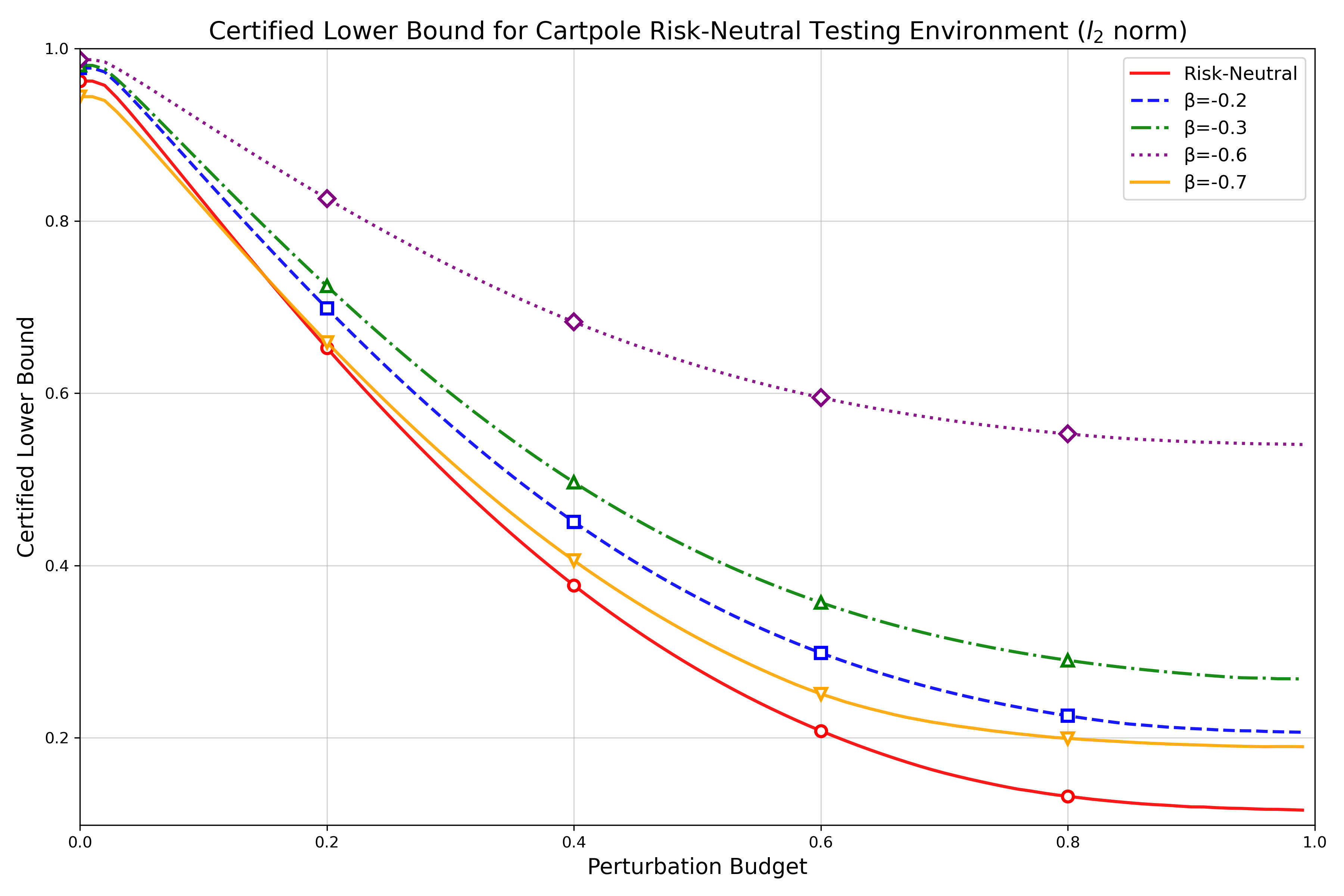}%
        \label{fig: cartpole_risk_neutral_l2}%
    }
    \hfill
    \subfloat[$\beta_{\text{test}}=-0.1$ (risk-averse), $l_{2}$-norm]{%
        \includegraphics[width=0.45\textwidth]{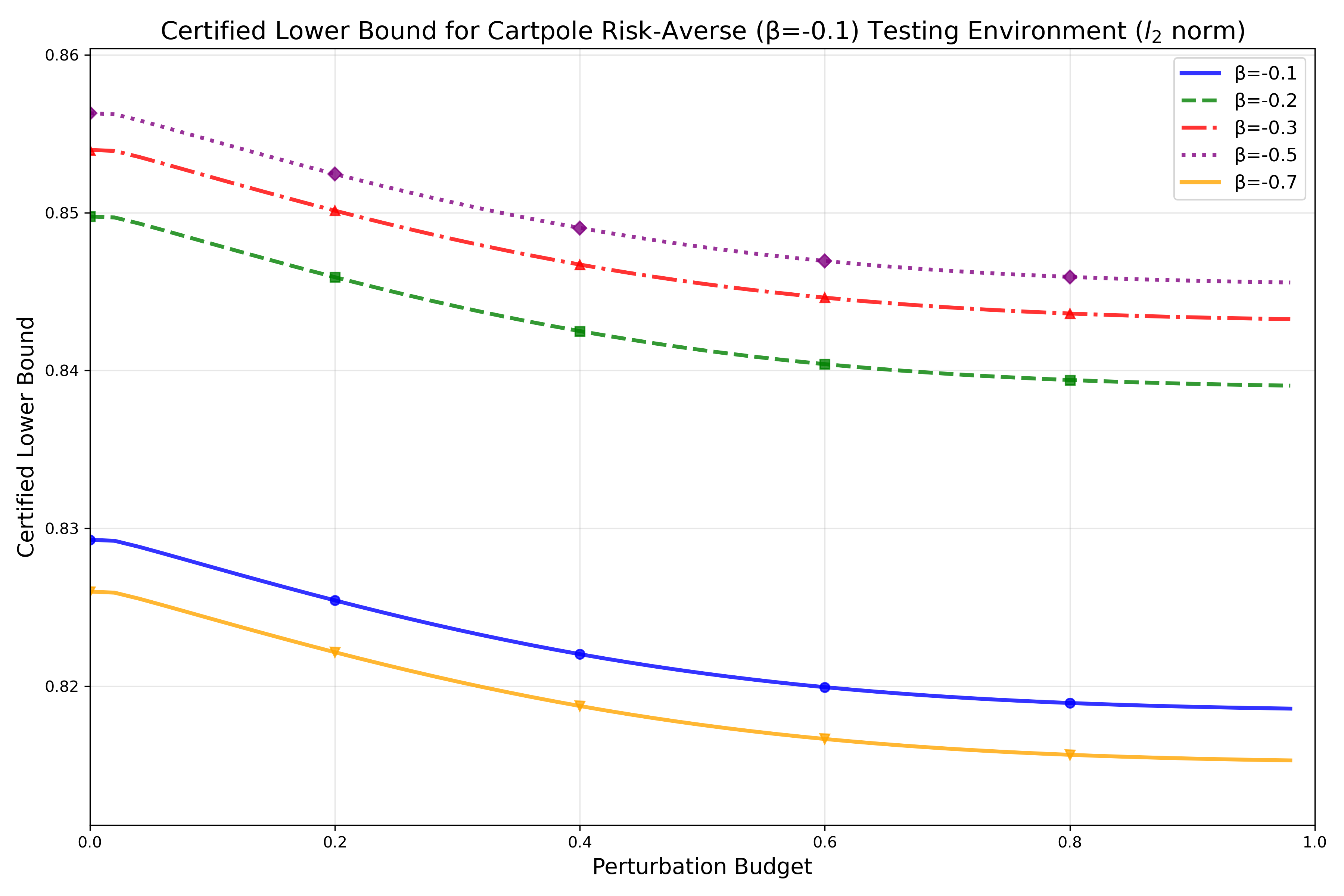}%
        \label{fig: cartpole_risk_averse_0.1_l2}%
    }

    \vspace{0.2cm}

    \subfloat[$\beta_{\text{test}}=-0.4$ (risk-averse), $l_{2}$-norm]{%
        \includegraphics[width=0.45\textwidth]{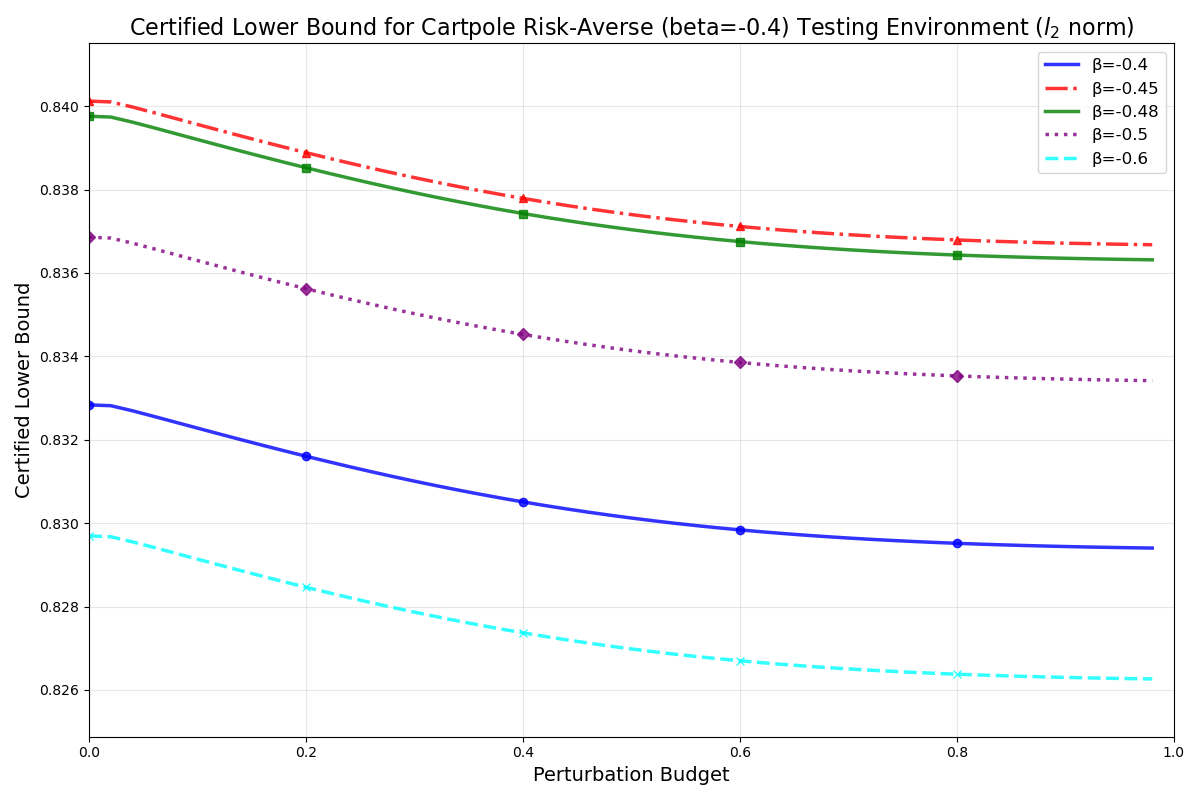}%
        \label{fig: cartpole_risk_averse_0.4_l2}%
    }
    \hfill
    \subfloat[$\beta_{\text{test}}=-0.6$ (risk-averse), $l_{2}$-norm]{%
        \includegraphics[width=0.45\textwidth]{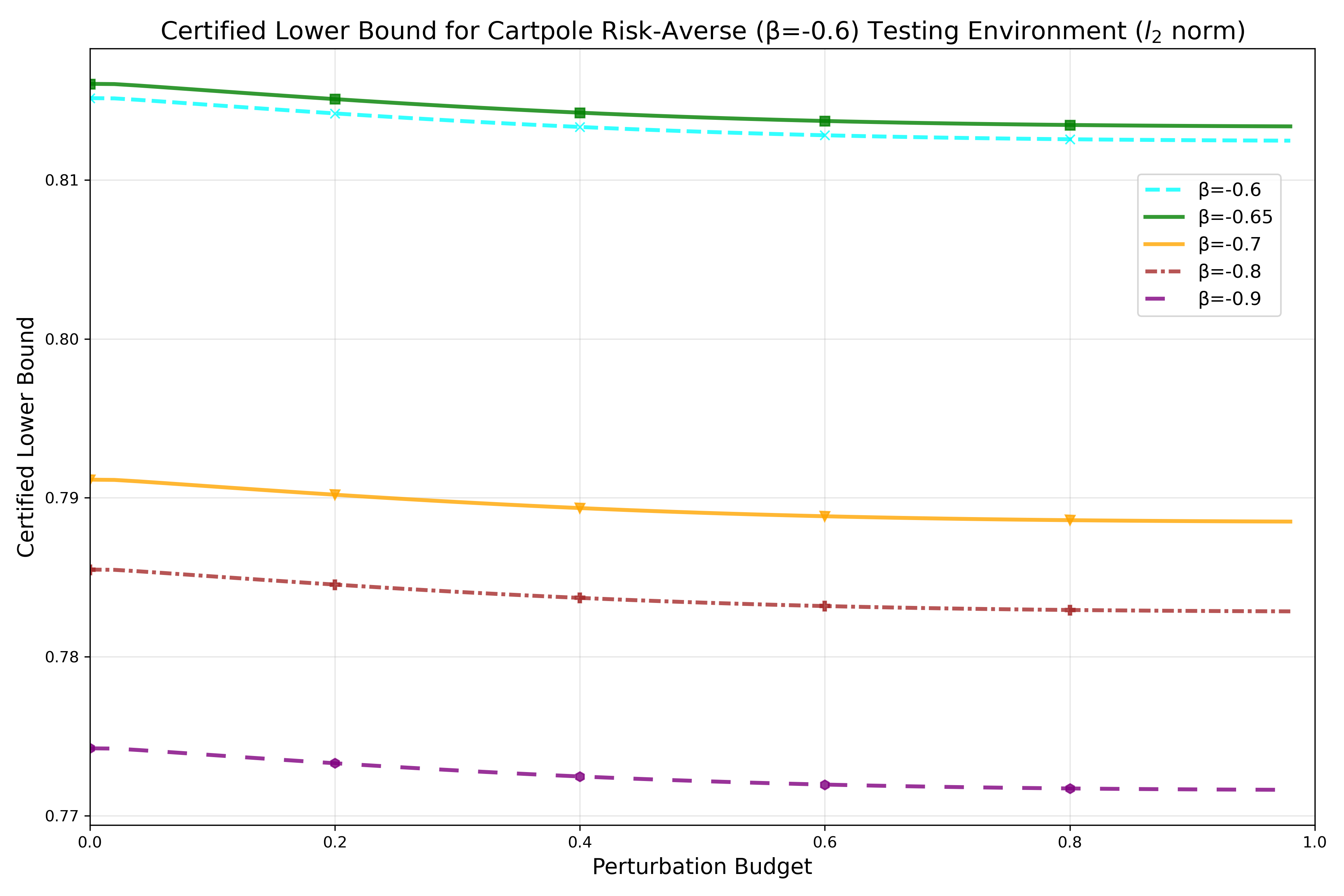}%
        \label{fig: atari_riskaverse0.7 result}%
    }

    \vspace{0.2cm}

    \subfloat[$\beta_{\text{test}} \to 0$ (risk-neutral), $l_{1}$-norm]{%
        \includegraphics[width=0.45\textwidth]{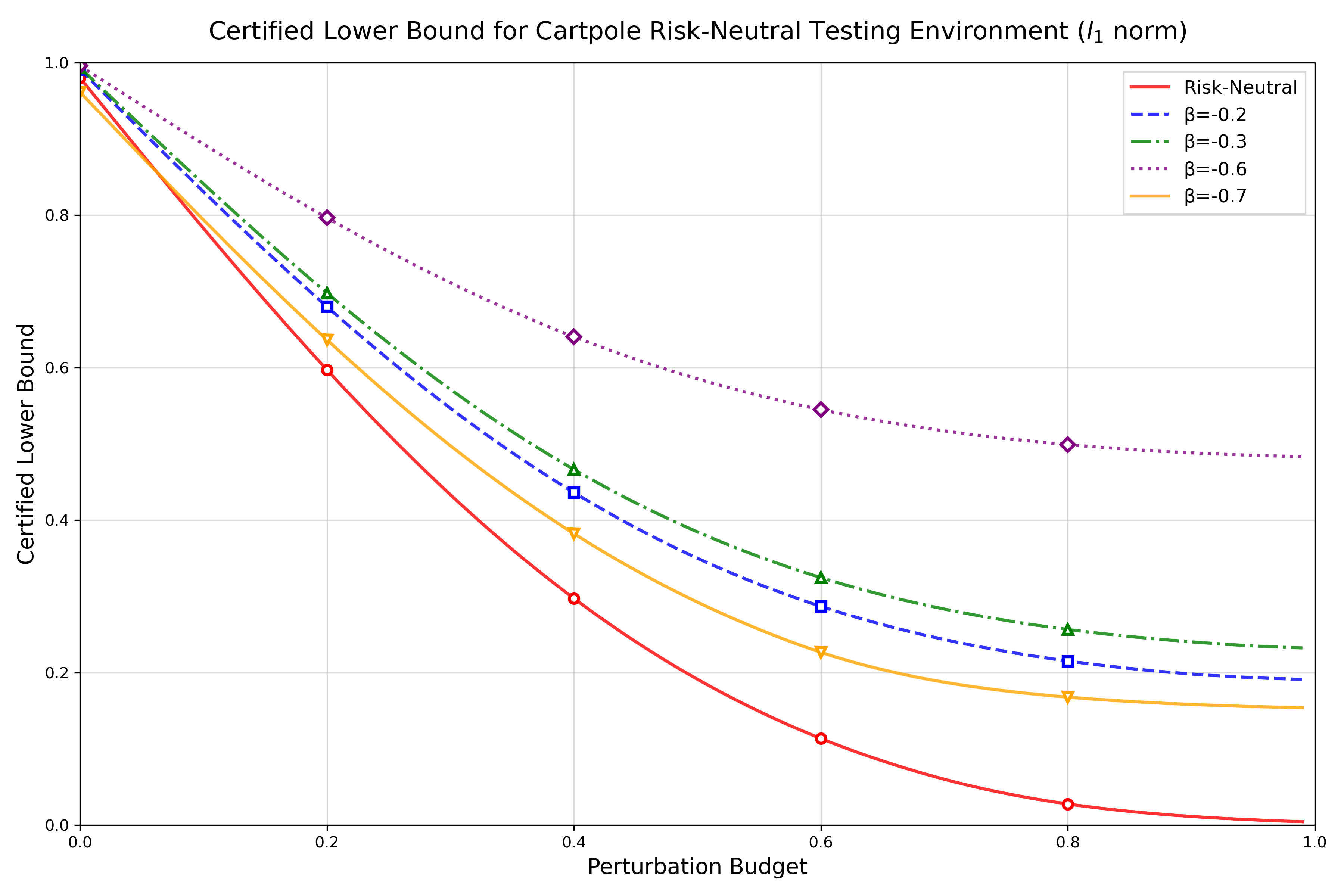}%
        \label{fig: cartpole_risk_neutral_l1}%
    }
    \hfill
    \subfloat[$\beta_{\text{test}}=-0.1$ (risk-averse), $l_{1}$-norm]{%
        \includegraphics[width=0.45\textwidth]{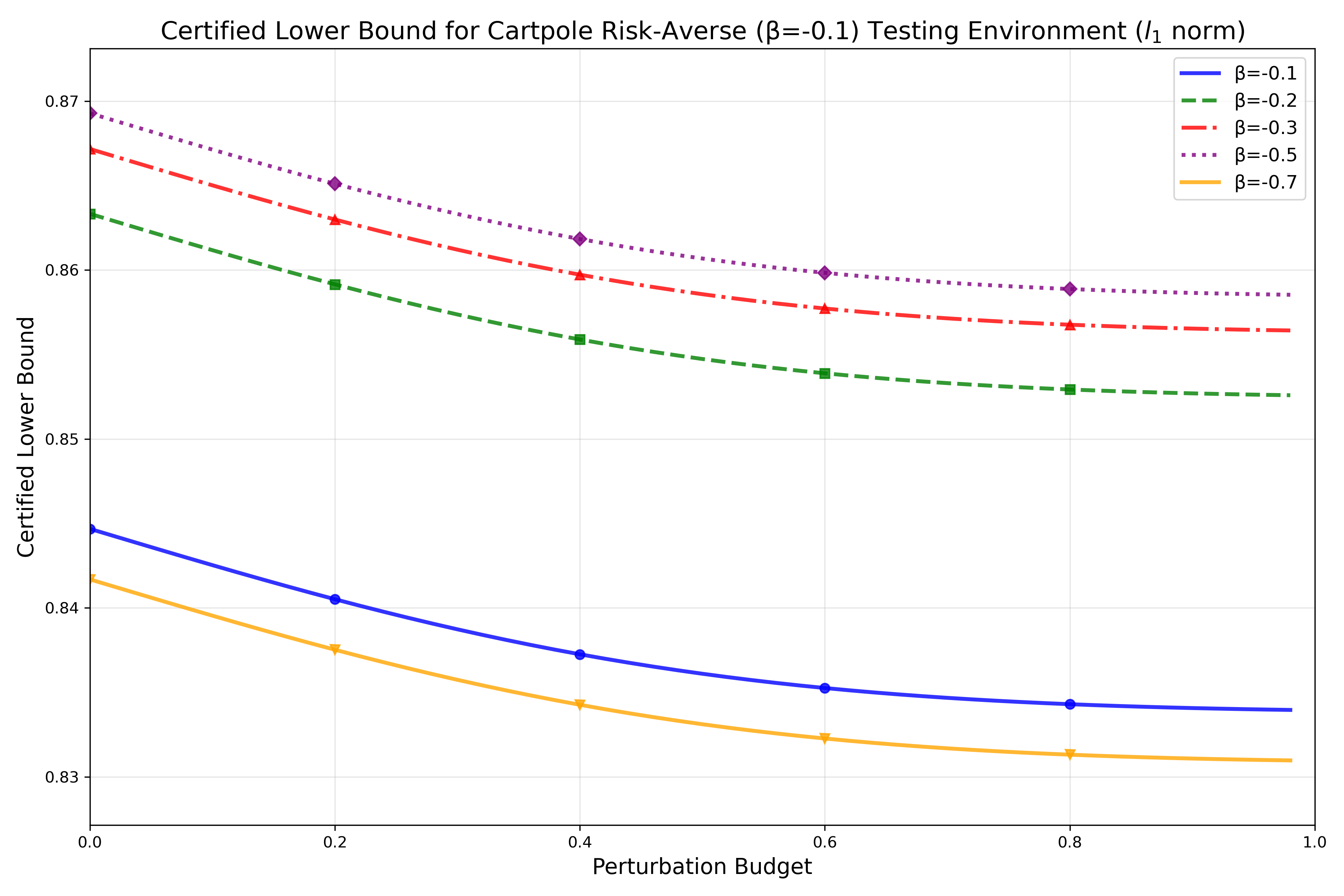}%
        \label{fig: cartpole_risk_averse_0.1_l1}%
    }

    \vspace{0.2cm}

    \subfloat[$\beta_{\text{test}}=-0.4$ (risk-averse), $l_{1}$-norm]{%
        \includegraphics[width=0.45\textwidth]{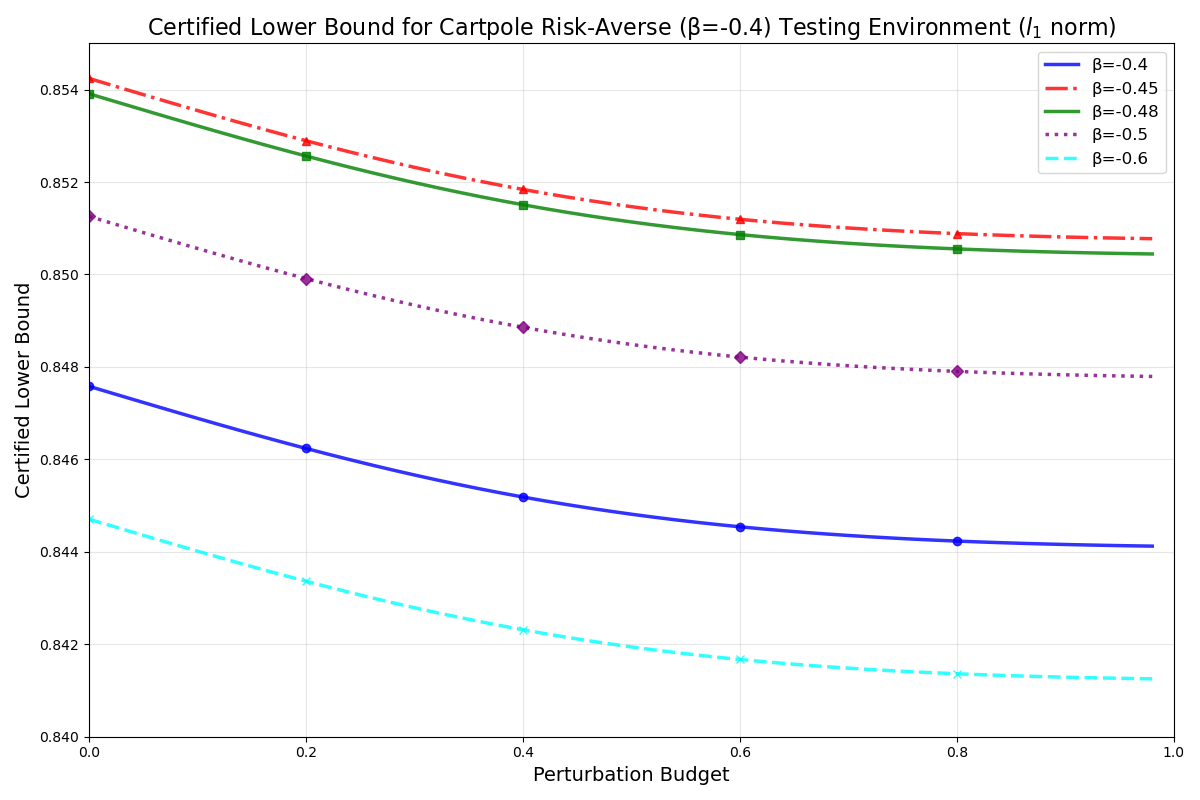}%
        \label{fig: cartpole_risk_averse_0.4_l1}%
    }
    \hfill
    \subfloat[$\beta_{\text{test}}=-0.6$ (risk-averse), $l_{1}$-norm]{%
        \includegraphics[width=0.45\textwidth]{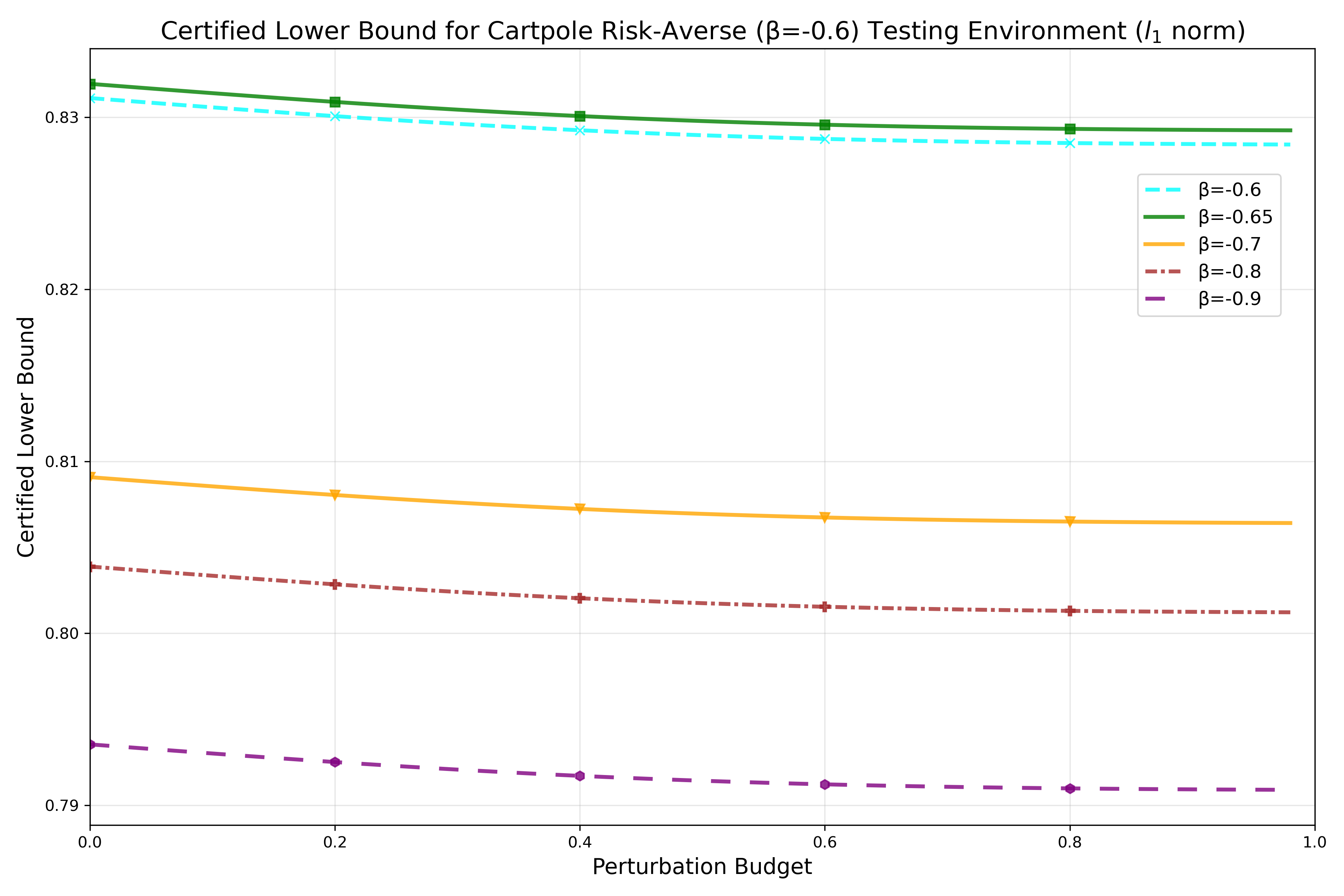}%
        \label{fig: cartpole_risk_averse_0.6_l1}%
    }

    \caption{Certified lower bounds for CartPole with risk-neutral and risk-averse testing environments under $l_{2}$-norm (top two rows) and $l_1$-norm (bottom two rows) perturbations.}
    \label{fig: cartpole}
\end{figure*}

Figures \ref{fig: ln_risk_averse_0.2_l2} to \ref{fig: ln_risk_averse_0.6_l2} and \ref{fig: ln_risk_averse_0.2_l1} to \ref{fig: ln_risk_averse_0.6_l1} present the certified lower bounds evaluated in risk-averse testing scenarios under $l_2$ and $l_1$ norm-bounded perturbations, respectively. 
Among these subfigures, 
we observe that the certified lower bounds initially increase and subsequently decrease as $\beta_{\text{train}}$ decreases, which is consistent with the pattern observed in the risk-neutral testing case, where relatively moderate risk aversion tends to improve robustness, while excessive aversion can lead to performance degradation.

We also consider the CartPole environment from the classic control suite of OpenAI Gym, in which an agent needs to balance a pole by moving a cart left or right \citep{towers2024gymnasium}. In the CartPole environment, we perturb the continuous state observations using Gaussian noise with zero mean and a standard deviation of 0.2 to simulate adversarial state uncertainty. The CartPole experiment follows the same setup as Lunar Lander. Specifically, we compare the certified lower bounds of various RL policies under both risk-neutral and risk-averse evaluation settings, measured with respect to adversarial perturbations bounded by the $l_2$ and $l_1$ norms. The results are presented in Figure \ref{fig: cartpole}. Figure \ref{fig: cartpole} shows similar trends to those observed in Figure \ref{fig: lunar lander result}, supporting the consistency of our findings.

\subsection{Machine replacement problem}

In order to gain more insights on why increasing risk aversion during training can lead to more robust policies and improved testing performance, we next evaluate our certification framework on the machine replacement problem \citep{puterman1994markov}.

We consider a machine with a continuous state space $\mcs=[0,10)$, where smaller state values correspond to healthier operating conditions and larger values indicate increasing levels of deterioration. Specifically, the interval $[0,1)$ represents a like-new condition, whereas states closer to the upper bound correspond to severely degraded conditions. For computational tractability, the continuous state space is uniformly partitioned into 10 intervals: $[0,1), [1,2), \ldots, [9,10)$.
Each continuous state \( s \in S \) is discretized using the floor function \( \lfloor s \rfloor \), which denotes the greatest integer less than or equal to \( s \). This yields a finite set of discretized degradation levels indexed from 0 to 9.

At each decision step, the agent makes a binary decision: to do nothing $(a=0)$ or to perform maintenance $(a=1)$.
The last degradation interval $[9,10)$ is an absorbing state, where the agent is required to perform maintenance. The cost function is defined as:
\[
c(s, a) = 
\begin{cases}
0, & a = 0, s\in [0,9)\\
2, & a = 1, s \in [0,9)\\
5, & a = 1,\, s \in [9,10).
\end{cases}
\]
When action \( a = 1 \) is taken, the machine is fully repaired and returns to the like-new interval \([0, 1)\) with probability 1. When doing-nothing action (\( a = 0 \)) is selected, the system evolves according to a Poisson deterioration process. Let \( i = \lfloor s_t \rfloor \) denote the discretized degradation level of the current state \( s_t \).
The transition probability to the next discretized state \( j \geq i \) is given by a truncated Poisson distribution with a state-dependent rate parameter \( \lambda_i \):
\begin{equation*}
    \mathbb{P}(\lfloor s_{t+1} \rfloor = j \mid \lfloor s_t \rfloor = i, a_t = 0) =
    \frac{\lambda_i^{j-i} e^{-\lambda_i}}{(j-i)! Z_i}, 
\end{equation*}
where \( j = i, i+1, \dots, 8 \) and \( Z_i = \sum_{k=0}^{8 - i} \frac{\lambda_i^k e^{-\lambda_i}}{k!} \) is the normalization constant ensuring a valid probability distribution. 
The rate parameter \( \lambda_i \) varies with the discretized degradation level \( i \) and increases with the degree of deterioration, taking values \( \{0.1, 0.2, 0.3, 0.4, 0.5, 0.6, 0.8, 1.0, 1.5\} \) for \( i = 0, 1, \ldots, 8 \), respectively.
The discretized state $\lfloor s\rfloor=9$ is treated as an absorbing state under the action $a=0$. 

The machine replacement problem traditionally aims to minimize the expected cumulative cost. To maintain consistency with the previous experiments, we reformulate the objective as maximizing the cumulative negative cost, interpreted as a reward in the RL framework.

In line with the previous experiments, we assess the certified lower bounds across different testing conditions to investigate the impact of training-time risk preferences on certified performance. In these evaluations, the underlying state is perturbed by zero-mean Gaussian noise with a standard deviation of 0.5. Figures \ref{fig: mr_risk_neutral_l2} and \ref{fig: mr_risk_neutral_l1} present the certified lower bounds of the policies evaluated under a risk-neutral testing condition, with adversarial perturbations bounded by the $l_2$ and $l_1$ norms, respectively.
We observe that policies trained with risk-averse objectives tend to achieve higher certified lower bounds than risk-neutral ones, particularly under larger perturbation budgets. Moreover, as the degree of risk aversion decreases, the certified lower bounds initially increase but decrease at
$\beta_{\text{train}}=-0.8$, indicating a non-monotonic relationship between training-time risk aversion and testing-time certification performance. This pattern can be explained by the structural property of the optimal policy in the machine replacement problem. Specifically, the optimal policy is characterized by a threshold $s^{*}$: the agent performs maintenance $(a=1)$ when $\lfloor s \rfloor \geq s^{*}$ and does nothing $(a=0)$ otherwise \citep{puterman1994markov}. When the policy is trained under a risk-averse objective, the threshold tends to be lower, leading the agent to select the repair action at earlier stages of degradation. Such behavior enhances policy robustness under perturbed observations, as the agent is more likely to make the same maintenance decision even when the observed state deviates from the true underlying state. As a result, the certified lower bounds of risk-averse policies can outperform those of risk-neutral ones, particularly under testing conditions with larger perturbation budgets. However, when the risk-aversion level becomes excessively strong, the learned policy turns overly conservative, initiating maintenance too early and thereby causing the certified lower bounds to deteriorate during testing.

\begin{figure*}[htbp]
    \centering

    \subfloat[$\beta_{\text{test}} \to 0$ (risk-neutral), $l_{2}$-norm]{%
        \includegraphics[width=0.45\textwidth]{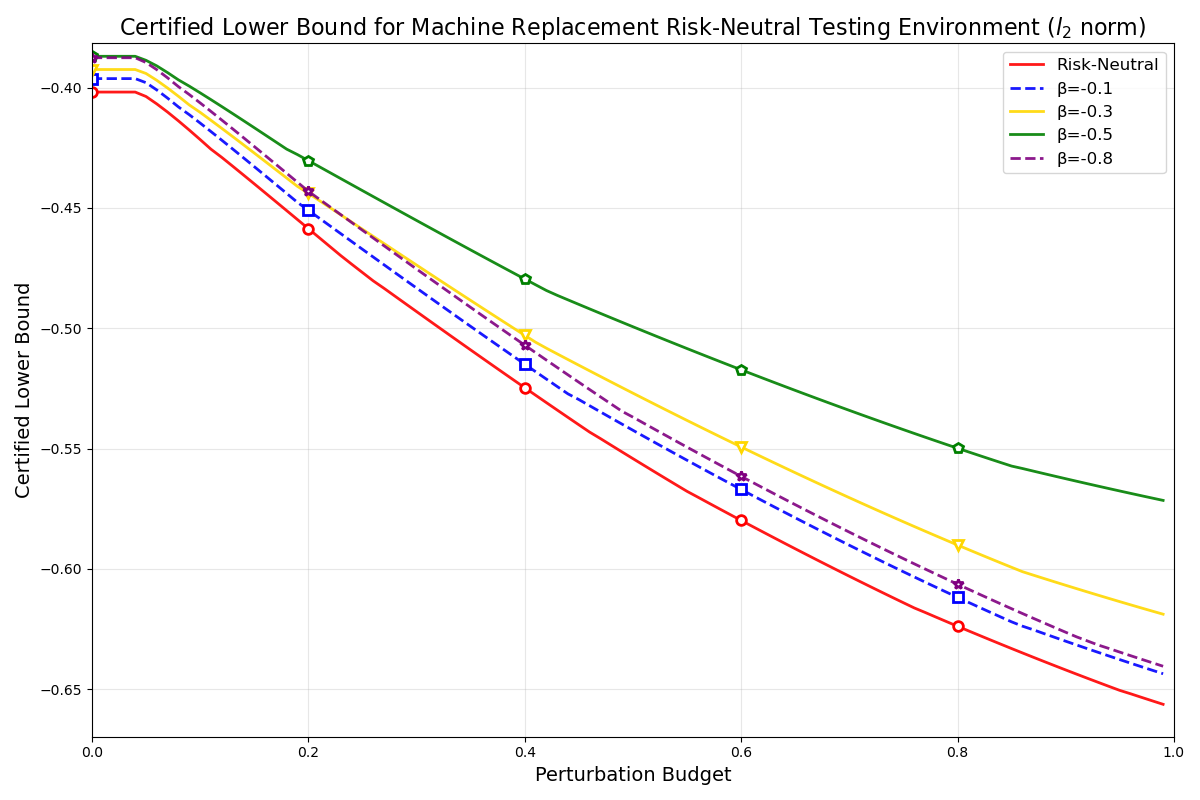}%
        \label{fig: mr_risk_neutral_l2}
    }
    \hfill
    \subfloat[$\beta_{\text{test}}=-0.1$ (risk-averse), $l_{2}$-norm]{%
        \includegraphics[width=0.45\textwidth]{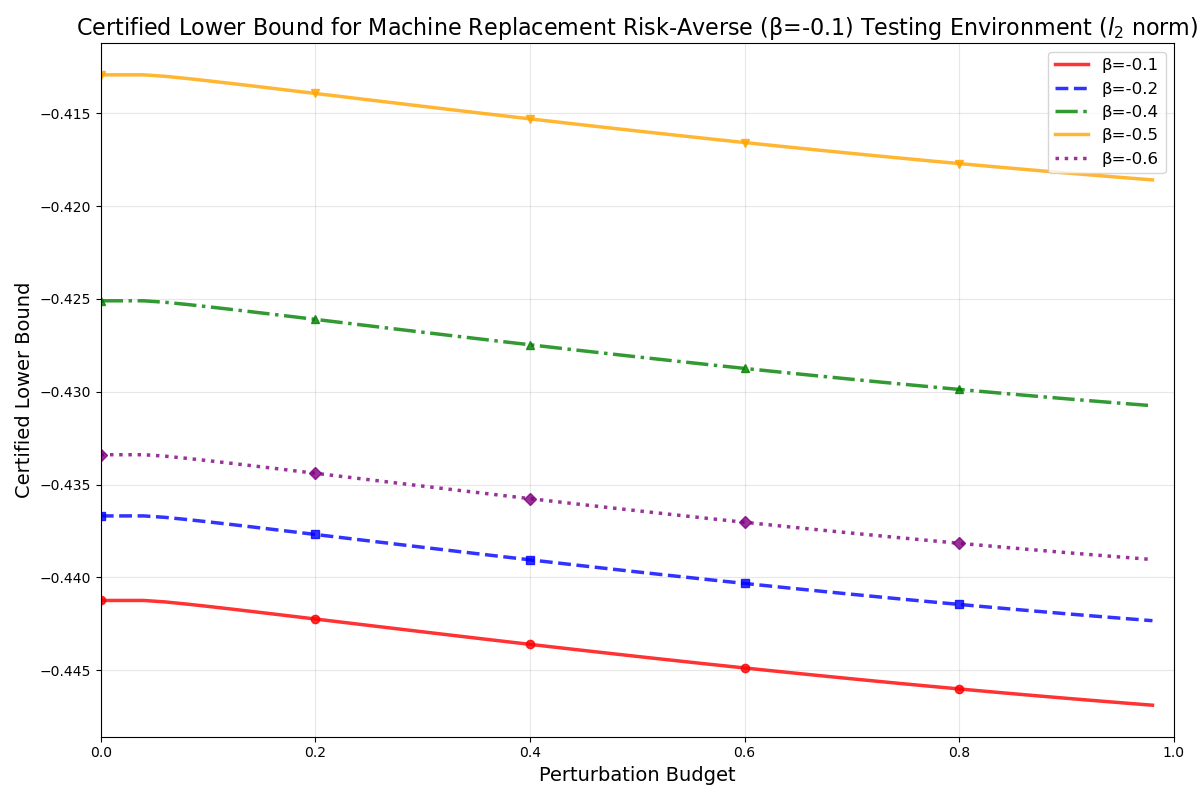}%
        \label{fig: mr_risk_averse_0.1_l2}
    }

    \vspace{0.2cm}

    \subfloat[$\beta_{\text{test}}=-0.3$ (risk-averse), $l_{2}$-norm]{%
        \includegraphics[width=0.45\textwidth]{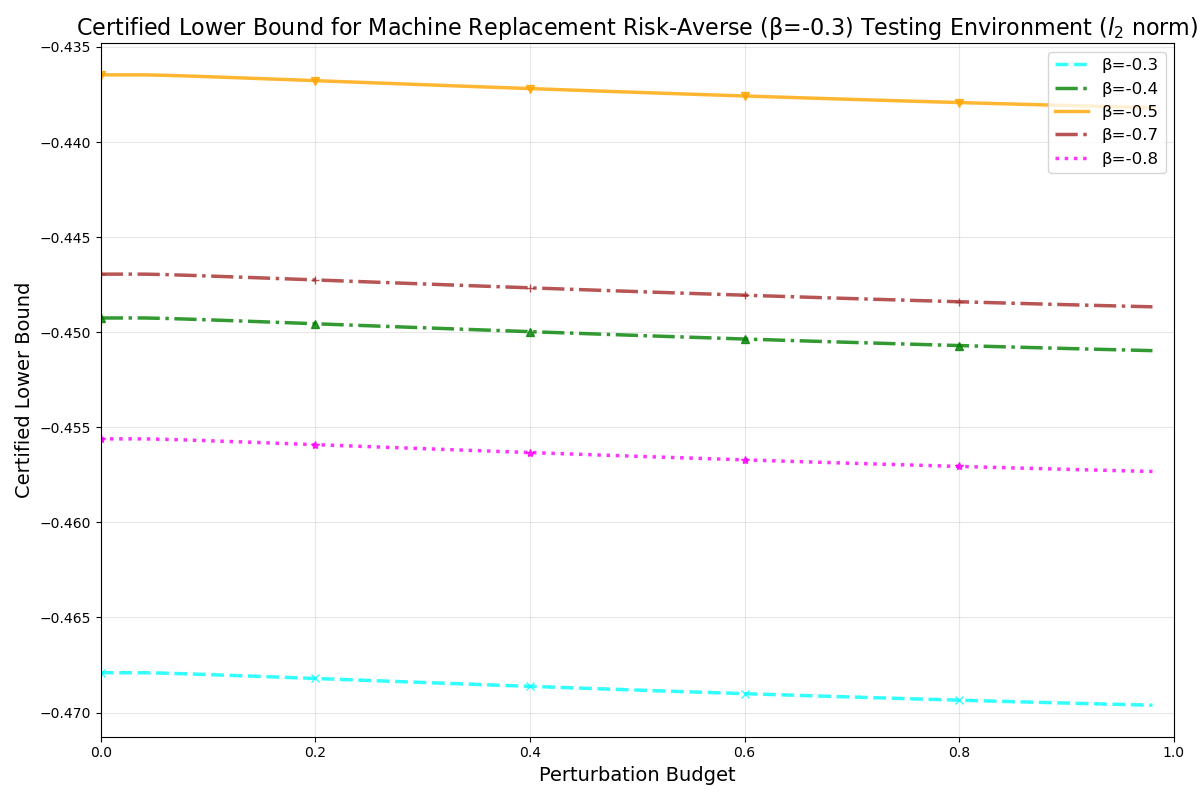}%
        \label{fig: mr_risk_averse_0.3_l2}
    }
    \hfill
    \subfloat[$\beta_{\text{test}}=-0.6$ (risk-averse), $l_{2}$-norm]{%
        \includegraphics[width=0.45\textwidth]{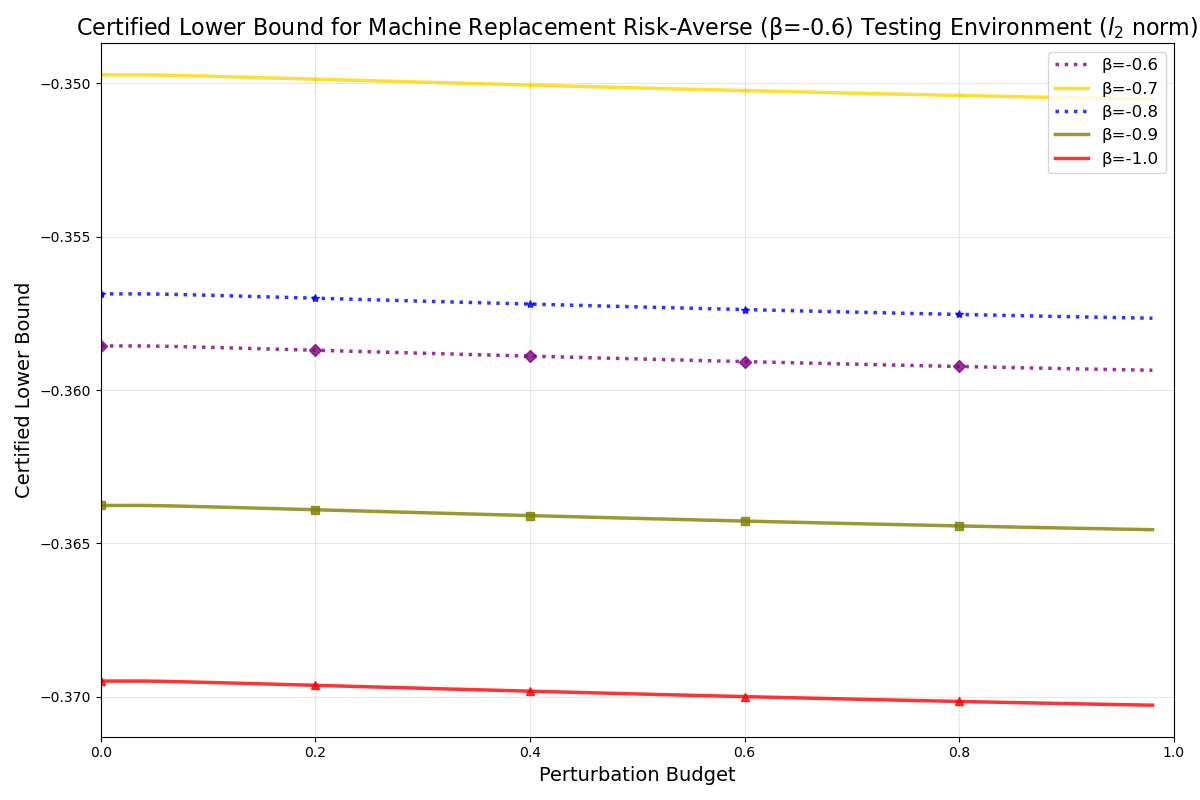}%
        \label{fig: mr_risk_averse_0.6_l2}
    }

    \vspace{0.2cm}

    \subfloat[$\beta_{\text{test}} \to 0$ (risk-neutral), $l_{1}$-norm]{%
        \includegraphics[width=0.45\textwidth]{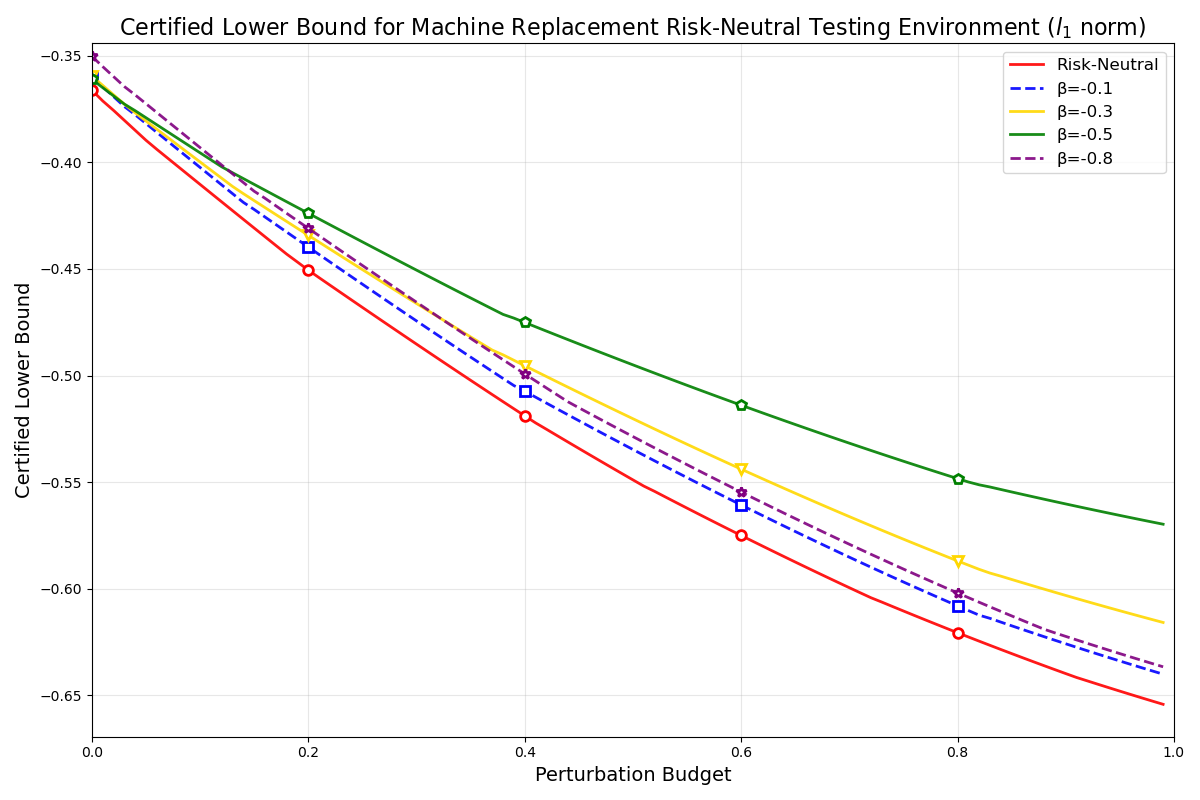}%
        \label{fig: mr_risk_neutral_l1}
    }
    \hfill
    \subfloat[$\beta_{\text{test}}=-0.1$ (risk-averse), $l_{1}$-norm]{%
        \includegraphics[width=0.45\textwidth]{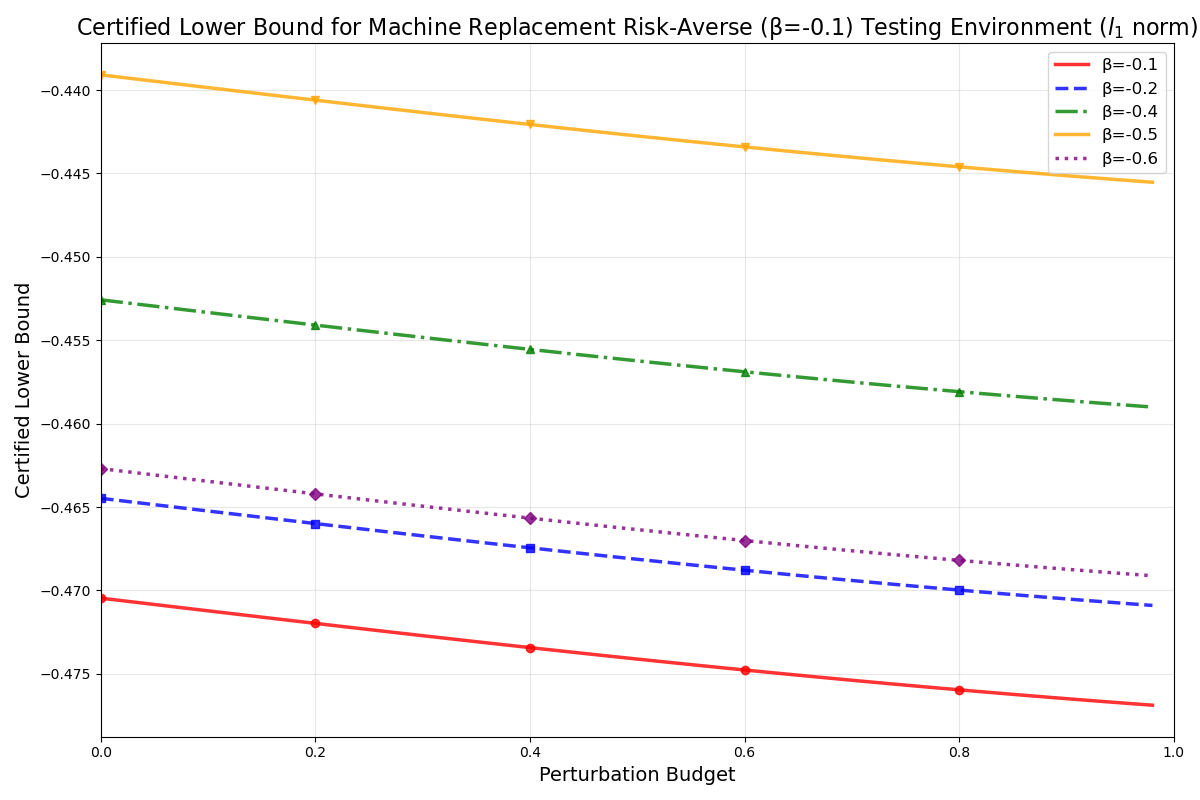}%
        \label{fig: mr_risk_averse_0.1_l1}
    }

    \vspace{0.2cm}

    \subfloat[$\beta_{\text{test}}=-0.3$ (risk-averse), $l_{1}$-norm]{%
        \includegraphics[width=0.45\textwidth]{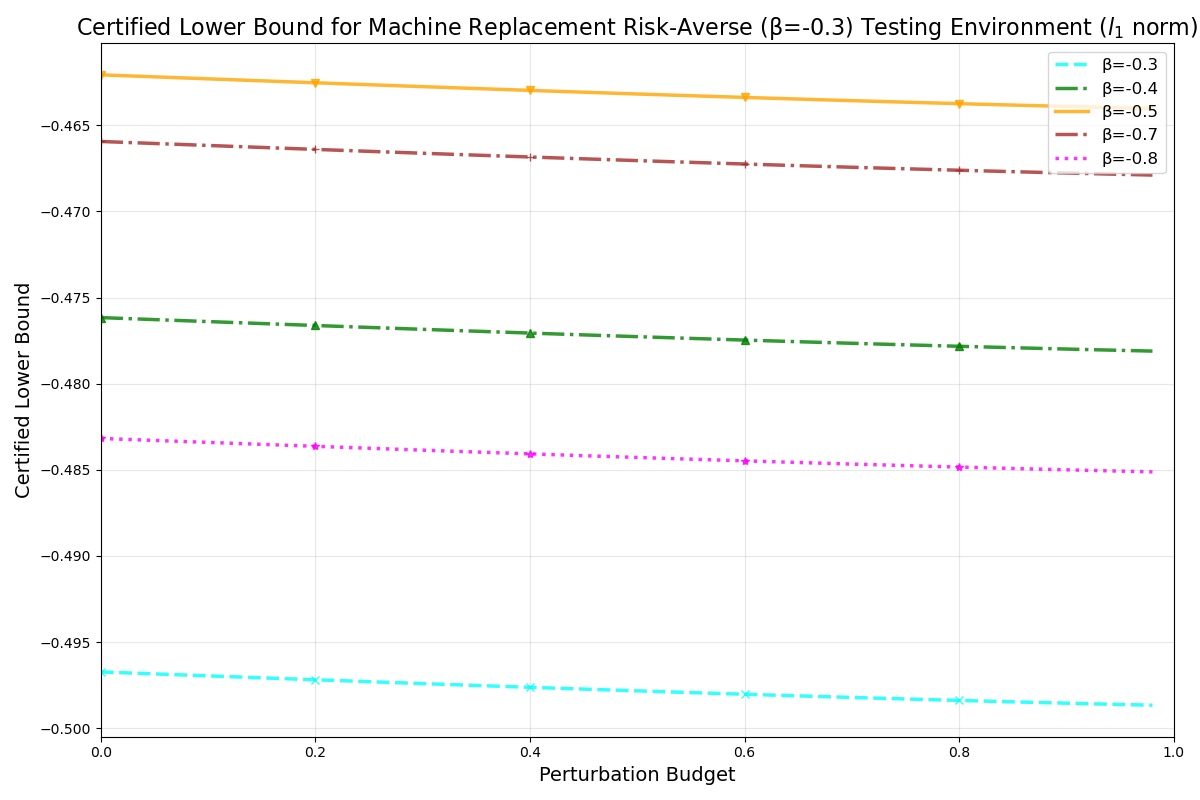}%
        \label{fig: mr_risk_averse_0.3_l1}
    }
    \hfill
    \subfloat[$\beta_{\text{test}}=-0.6$ (risk-averse), $l_{1}$-norm]{%
        \includegraphics[width=0.45\textwidth]{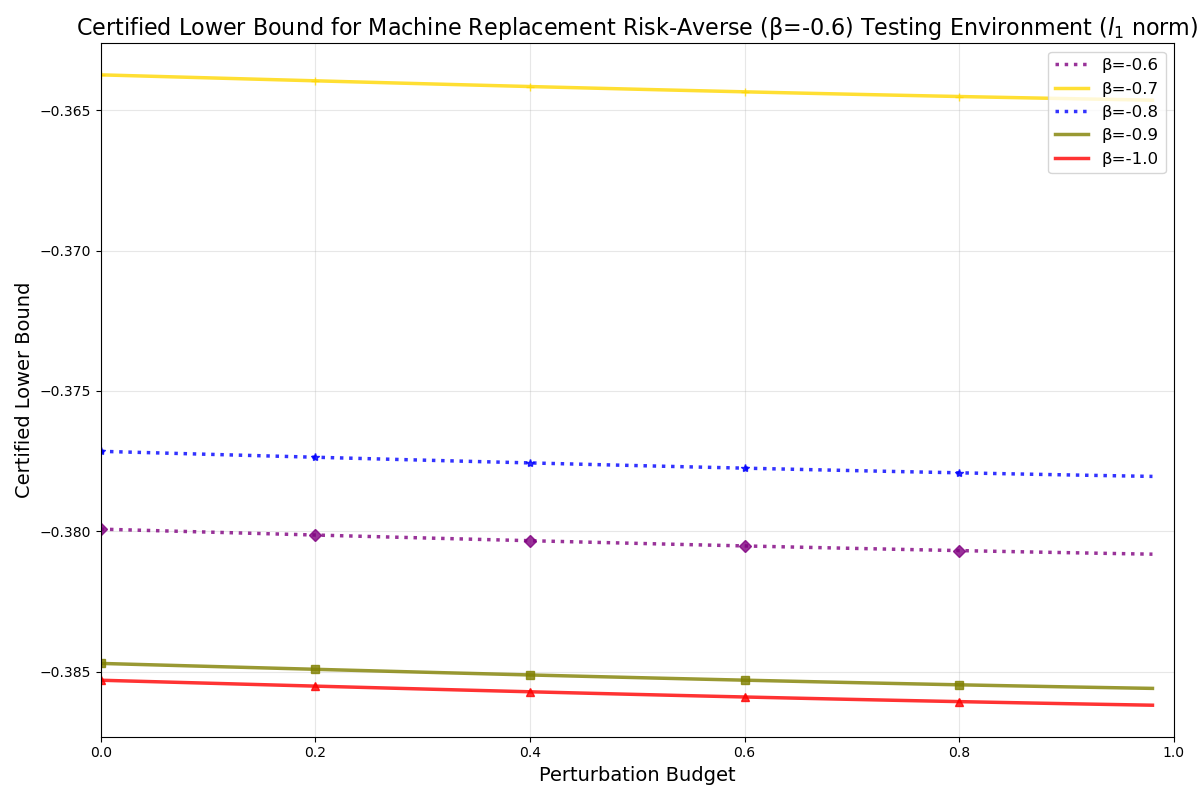}%
        \label{fig: mr_risk_averse_0.6_l1}
    }

    \caption{Certified lower bounds for the machine replacement problem with risk-neutral and risk-averse testing environments under $l_{2}$-norm (top two rows) and $l_1$-norm (bottom two rows) perturbations.}
    \label{fig: machine replacement risk-averse result}
\end{figure*}

Figures \ref{fig: mr_risk_averse_0.1_l2} to \ref{fig: mr_risk_averse_0.6_l2} and \ref{fig: mr_risk_averse_0.1_l1} to \ref{fig: mr_risk_averse_0.6_l1} present the certified lower bounds of policies trained under varying levels of risk aversion, evaluated in testing environments with parameters $\beta_{\text{test}}=-0.1, -0.3$ and $-0.6$. We observe that, under both \( l_2 \)- and \( l_1 \)-bounded perturbations within the same testing environment, the certified lower bounds initially increase and then decrease as the training-time risk parameter \( \beta_{\text{train}} \) becomes more negative. This non-monotonic trend is consistent with the patterns previously observed in the Lunar Lander and CartPole environments. 

\section{Conclusion}
In this paper, we propose a novel certification framework for evaluating the robustness of RL policies under state adversarial perturbations. 
Specifically, we formulate the certification objective as the expectation of exponential utility returns to reflect the agent’s risk preference under $l_{p}$-norm bounded perturbations ($1\leq p < \infty$). 
By introducing a $\phi$-divergence–based constraint on the perturbation budget, we construct a convex relaxation of the risk-sensitive certification objective and derive its tractable dual formulation.
We conduct extensive experiments on both OpenAI Gym environments and a machine replacement task. In each setting, we train policies using a range of training-time risk aversion levels ($\beta_{\text{train}}$). We then evaluate their certified lower bounds under different evaluation-time parameters ($\beta_{\text{test}}$) to analyze how risk preferences influence certified robustness.
Our results show that risk-averse training generally results in policies with higher certified lower bounds than risk-neutral training, especially under larger perturbation budgets. Moreover, in both risk-neutral and risk-averse evaluation settings, increasing risk aversion during training tends to improve certified robustness initially, while excessive risk aversion will ultimately degrade certification performance due to overly conservative policies.

\begin{appendices}

\section{Equivalence Between Trajectory Distribution Divergence and Observation Distribution Divergence}\label{appendix: observation reduction}

In this appendix, we reduce the $\phi$-divergence between trajectory distributions $q(\tau)$ and $p(\tau)$, with $\tau = (s_1, o_1, a_1, \ldots, s_T, o_T, a_T)$, to the $\phi$-divergence between observation distributions $\mu(\cdot \mid s_1 + \delta_1)$ and $\mu(\cdot \mid s_1)$, given a fixed initial state $s_1$ and assuming that only $s_1$ is perturbed by $\delta_1$.

\begin{align}
   & D_{\phi}(q(\tau)||p(\tau)) \nonumber\\
   &=\int_{\tau}\phi\left(\frac{q(\tau)}{p(\tau)}\right)dp(\tau) \nonumber\\
   &=\int_{\tau}\phi\left(\frac{\mu(o_1|s_1+\delta_1)}{\mu(o_1|s_1)} \times \frac{\bcancel{\pi(a_1|o_1)}}{\bcancel{\pi(a_1|o_1)}} \times\prod_{t=2}^{T} \frac{\bcancel{P(s_{t}|s_{t-1}, a_{t-1})\mu(o_t|s_t)\pi(a_t|o_t)}}{\bcancel{P(s_{t}|s_{t-1}, a_{t-1})\mu(o_t|s_t)\pi(a_t|o_t))}} \right)dp(\tau) \label{explain p and q} \\
   &= \int_{o_1}\phi\left(\frac{\mu(o_1|s_1+\delta_1)}{\mu(o_1|s_1)}\right) \mu(o_1|s_1)  \nonumber\\
   &\quad \times \left( \int_{a_1,s_2,o_2,a_2,\ldots,s_{T},o_{T}, a_{T}} P(a_1,\ldots,s_{T},o_{T}, a_{T}|o_1,s_1) da_1 ds_{2}do_{2}da_{2}\ldots ds_{T}do_{T}da_{T} \right) do_1 \nonumber\\
   &= \int_{o_1}\mu(o_1|s_1) \phi\left(\frac{\mu(o_1|s_1+\delta_1)}{\mu(o_1|s_1)}\right) do_1 \label{explain integral} \\
   &= D_{\phi}(\mu(\cdot|s_1+\delta_1)||\mu(\cdot|s_1)),\nonumber
\end{align}
where \eqref{explain p and q} is due to the expanded form \eqref{eq: expression of q} and \eqref{eq: expression of p} of the trajectory distributions $q(\tau)$ and $p(\tau)$ respectively. 
\eqref{explain integral} follows from marginalizing the integral over all trajectory variables to retain only the initial observation $o_1$.


\section{Proof of Proposition \ref{prop1}}\label{appendix: proof of l1}
We first prove that the total variation (TV) divergence between the two Gaussian distributions 
$\mu = \mathcal{N}(s_1 + \delta, \sigma^2 I_d)$ and 
$\nu = \mathcal{N}(s_1, \sigma^2 I_d)$ 
is equal to the TV divergence between their corresponding standardized distributions 
$\mu_0 = \mathcal{N}(\tilde{\delta}, I_d)$ and 
$\nu_0 = \mathcal{N}(0, I_d)$, 
where $\tilde{\delta} = \delta / \sigma$.

By definition, the TV divergence between distributions $\mu$ and $\nu$ is given by
$$D_{TV}(\mu \| \nu) = \frac{1}{2} \int_{\mathbb{R}^d} |p(x) - q(x)| dx,$$
where $p(x) = \frac{1}{(2\pi\sigma^2)^{d/2}} \exp\left(-\frac{1}{2\sigma^2}\|x - (s_1 + \delta)\|^2\right)$ and $q(x) = \frac{1}{(2\pi\sigma^2)^{d/2}} \exp\left(-\frac{1}{2\sigma^2}\|x - s_1\|^2\right)$ denote the probability density functions of the Gaussian distributions $\mu$ and $\nu$, respectively. 

Applying the substitution $y = \frac{x - s_1}{\sigma}$ with $dx = \sigma^d dy$, we obtain:
\begin{align}
& D_{TV}(\mu \| \nu) \\
&= \frac{1}{2} \int_{\mathbb{R}^d} \left|\frac{1}{(2\pi\sigma^2)^{d/2}} \exp\left(-\frac{1}{2\sigma^2}\|\sigma y + s_1 - (s_1 + \delta)\|^2\right) \right. \nonumber \\
&\quad\quad\left. - \frac{1}{(2\pi\sigma^2)^{d/2}} \exp\left(-\frac{1}{2\sigma^2}\|\sigma y + s_1 - s_1\|^2\right)\right| \sigma^d dy \nonumber \\
&= D_{TV}(\mu_0 \| \nu_0). \nonumber
\end{align}

To evaluate the TV divergence $D_{TV}(\mu_0 \| \nu_0)$, we 
analyze the sign of the difference $p_0(x) - q_0(x)$, where $p_0$ and $q_0$ denotes the probability density functions of $\mu_0$ and $\nu_0$ respectively. Computing the density ratio yields:
\begin{align}
\frac{p_0(x)}{q_0(x)}&= \frac{\frac{1}{(2\pi)^{d/2}} \exp\left(-\frac{1}{2}\|x - \tilde{\delta}\|^2\right)}{\frac{1}{(2\pi)^{d/2}} \exp\left(-\frac{1}{2}\|x\|^2\right)} \nonumber \\
&= \exp\left(-\frac{1}{2}\|x - \tilde{\delta}\|^2 + \frac{1}{2}\|x\|^2\right)  \nonumber \\
&= \exp\left(-\frac{1}{2}(\|x\|^2 - 2x^T\tilde{\delta} + \|\tilde{\delta}\|^2) + \frac{1}{2}\|x\|^2\right)  \nonumber\\
&= \exp\left(x^T\tilde{\delta} - \frac{1}{2}\|\tilde{\delta}\|^2\right).  \nonumber
\end{align}
It follows that $p_0(x) > q_0(x)$ if and only if $x^T\tilde{\delta} > \frac{1}{2}\|\tilde{\delta}\|^2$, which defines a half-space $H = \{x \in \mathbb{R}^d : x^T\tilde{\delta} > \frac{1}{2}\|\tilde{\delta}\|^2\}$. 
As a result, the TV divergence can be expressed as
\begin{align}
D_{TV}(\mu_0 \| \nu_0) &= \frac{1}{2} \int_{\mathbb{R}^d} |p_0(x) - q_0(x)| dx \nonumber \\
&= \frac{1}{2} \left[ \int_H (p_0(x) - q_0(x)) dx + \int_{H^c} (q_0(x) - p_0(x)) dx \right] \nonumber \\
&= \int_H p_0(x) dx - \int_H q_0(x) dx \nonumber \\
&= P(H) - Q(H) \nonumber ,
\end{align}
where $P(H)$ and $Q(H)$ denote denote the probabilities of the half-space under
$\mu_0$ and $\nu_0$, respectively.

We next compute these probabilities by analyzing the distributions of linear combinations $X^\top \tilde{\delta}$ and $Y^\top \tilde{\delta}$ respectively, 
where $X=(X_1,\cdots,X_{d})\sim \mathcal{N}(\tilde{\delta}, I_d)$ and $Y=(Y_1,\cdots,Y_{d}) \sim \mathcal{N}(0, I_d)$. 
Since 
$X^\top \tilde{\delta} \sim \mathcal{N}(\|\tilde{\delta}\|^2, \|\tilde{\delta}\|^2)$ and $Y^\top \tilde{\delta} \sim \mathcal{N}(0, \|\tilde{\delta}\|^2)$, we obtain:
\[
P(H) = \mathbb{P}\left[X^\top \tilde{\delta} > \frac{\|\tilde{\delta}\|^2}{2} \right] 
= \Phi\left( \frac{\|\tilde{\delta}\| }{2} \right),
\]
and
\[
Q(H) = \mathbb{P}\left[Y^\top \tilde{\delta} > \frac{\|\tilde{\delta}\|^2}{2} \right] 
= 1 - \Phi\left( \frac{\|\tilde{\delta}\| }{2} \right),
\]
where $\Phi(\cdot)$ denotes the CDF of the standard normal distribution $\mathcal{N}(0,1)$.
Hence, the total variation is given by
\begin{equation}\label{eq: final relationship}
D_{TV}(\mu\| \nu) 
=D_{TV}(\mu_0 \| \nu_0)= P(H) - Q(H)= 2\Phi\left(\frac{\|\tilde{\delta}\|}{2}\right) - 1.
\end{equation}
Substituting $\tilde{\delta} = \frac{\delta}{\sigma}$ into \eqref{eq: final relationship} gives that
$$D_{TV}(\mu \| \nu) = 2\Phi\left(\frac{\|\delta\|_2}{2\sigma}\right) - 1.$$
Since $\|\delta\|_2 \leq \|\delta\|_1 \leq \epsilon$ and $\Phi$ is monotonically increasing, we obtain:
$$D_{TV}(\mu \| \nu) = 2\Phi\left(\frac{\|\delta\|_2}{2\sigma}\right) - 1 \leq 2\Phi\left(\frac{\|\delta\|_1}{2\sigma}\right) - 1 \leq 2\Phi\left(\frac{\epsilon}{2\sigma}\right) - 1.$$
This completes the proof. 

\end{appendices}


\bibliography{sn-bibliography}

\end{document}